\documentclass[11pt]{amsart}
\usepackage{hyperref}
\usepackage{halloweenmath}
\usepackage{amsmath,amsfonts,amssymb,amscd,amsthm,amsbsy,amsxtra,bbm,bm, epsf,calc,comment,appendix, xcolor}
\usepackage{color}
\usepackage{datetime}
\usepackage{latexsym}
\usepackage[english]{babel}
\usepackage{enumerate}
\usepackage{graphicx}
\usepackage{epsfig}
\usepackage{dsfont}
\usepackage{tikz}
\usepackage{caption}
\usepackage{subcaption}
\usepackage{algorithm}
\usepackage{algpseudocode}
\usepackage{dsfont}

\def\R{{\mathbb{R}}}

\def\Z{\mathcal{Z}}
\def\X{\mathcal{X}}
\def\Y{\mathcal{Y}}
\def\N{\mathbb{N}}
\def\II{{\rm I\kern-0.5exI}}
\def\III{{\rm I\kern-0.5exI\kern-0.5exI}}

\newcommand{\norm}[1]{\lVert #1 \rVert}

\newcommand{\veps}{\varepsilon}

\newcommand{\T}{\mathcal{T}}

\newcommand{\nc}{\normalcolor}

\DeclareMathOperator*{\argmax}{argmax}

\DeclareSymbolFont{bbold}{U}{bbold}{m}{n}
\DeclareSymbolFontAlphabet{\mathbbold}{bbold}

\newcommand{\cX}{\mathcal{X}}

\newcommand{\G}{\mathcal{G}}

\newcommand{\eps}{\varepsilon}

\newcommand{\spt}{\textup{spt}}

\numberwithin{equation}{section}
\newtheorem{theorem}{Theorem}[section]
\newtheorem{lemma}[theorem]{Lemma}
\newtheorem{proposition}[theorem]{Proposition}

\theoremstyle{remark}
\newtheorem{remark}[theorem]{Remark}

\theoremstyle{definition}
\newtheorem{definition}[theorem]{Definition}

\usepackage{amsthm}
\usepackage{hyperref}
\hypersetup
{colorlinks = true,
linkcolor = red, 
anchorcolor = red, 
citecolor = blue, 
filecolor = blue,
urlcolor = blue}
\usepackage{hypcap}
\usepackage{cleveref}
\usepackage{lipsum}

\usepackage{enumerate}
\usepackage{enumitem}

\title[OT Approach for Computing Adversarial Training Lower Bounds]{An Optimal Transport Approach for Computing Adversarial Training Lower Bounds in Multiclass Classification}

\author{Nicol\'as {Garc\'ia Trillos}}
\address{Department of Statistics, University of Wisconsin-Madison, 1300 University Avenue, Madison, WI 53706, USA.}
\email{garciatrillo@wisc.edu}
\author{Matt Jacobs}
\address{Department of Mathematics, UC Santa Barbara, 552 University Rd, Isla Vista, CA 93117}
\email{majaco@ucsb.edu}
\author{Jakwang Kim}
\address{Department of Mathematics, University of British Columbia, 1984 Mathematics Road, Vancouver, British Columbia, Canada, V6T 1Z2.}
\email{jakwang.kim@math.ubc.ca}
\author{Matthew Werenski}
\address{Department of Computer Science, Tufts University, 420 Joyce Cummings Center, 177 College Avenue,
Medford, MA 02155}
\email{matthew.werenski@tufts.edu}

\date{\today}


\begin{document}

\thanks{{\bf Acknowledgements:}
All authors' names are listed in alphabetic order by family name. This material is based upon work supported by the National Science Foundation under Grant Number DMS 1641020 and was started during the summer of 2022 as part of the AMS-MRC program \textit{Data Science at the Crossroads of Analysis, Geometry, and Topology}. NGT was supported by the NSF grant DMS-2236447. MJ is supported by NSF-DMS grant 2308217. JK thanks to PIMS Kantorovich Initiative supported through a PIMS PRN and NSF-DMS 2133244.}

\begin{abstract}
Despite the success of deep learning-based algorithms, it is widely known that neural networks may fail to be robust. A popular paradigm to enforce robustness is \emph{adversarial training} (AT), however, this introduces many computational and theoretical difficulties. Recent works have developed a connection between AT in the multiclass classification setting and multimarginal optimal transport (MOT), unlocking a new set of tools to study this problem. In this paper, we leverage the MOT connection to propose computationally tractable numerical algorithms for computing universal lower bounds on the optimal adversarial risk and identifying optimal classifiers. We propose two main algorithms based on linear programming (LP) and entropic regularization (Sinkhorn). Our key insight is that one can harmlessly \emph{truncate the higher order interactions} between classes, preventing the combinatorial run times typically encountered in MOT problems.  We validate these results with experiments on MNIST and CIFAR-$10$, which demonstrate the tractability of our approach.
\end{abstract}

\maketitle

	\section{Introduction}

	While neural networks have achieved state-of-the-art accuracy in classification problems, it is by now well known that networks trained with standard \emph{error risk minimization} (ERM) can be exceedingly brittle \cite{goodfellow2014explaining,chen2017zoo,qin2019imperceptible,cai2021zeroth}.  As a result, various works have suggested replacing ERM with alternative training procedures that enforce robustness.   In this paper, we focus on \emph{adversarial training} (AT), which converts standard risk minimization into a min-max problem where the learner is pitted against an adversary with the power to perturb the training data (see, e.g., \cite{madry2017towards, tramer2018ensemble, sinha2018certifiable}).  This provides a powerful defense against adversarial attacks at the cost of increased computation time and worse performance on clean data \cite{tsipras2018robustness}.    Hence, balancing robustness against efficiency of computation and clean accuracy is central to the effective deployment of adversarial training.

	A key hyperparameter in AT is the \emph{adversarial budget} $\varepsilon$, which controls how far the adversary is allowed to move each individual data point.  As $\varepsilon$ increases, an adversarially trained network will become more robust but will lose accuracy, as the learner is forced to be robust against stronger and stronger attacks. Due to the computational cost of adversarial training (one must solve a min-max problem rather than a pure min problem), hyperparameter tuning of the adversarial budget is very expensive.  As a result, there is a great practical need for 
	theory and algorithms that can
	predict good choices of $\varepsilon$ without requiring massive computation.  
	
	A recent body of work has attempted to address this issue by providing bounds on the gap between the optimal adversarial risk (with a given budget $\varepsilon$) and the optimal standard risk.  Indeed, provided that these bounds are reasonably tight and efficiently computable, they may provide a more efficient route for a practitioner to choose and tune $\varepsilon$.  
	
	Thus far, the literature has largely focused on the theoretical side of characterizing such bounds.  A number of authors have provided bounds for specific types of classifiers and neural network architectures \cite{ weng2018evaluating, 
		yin2019rademacher, khim2019adversarial}, though it is unclear whether any of these results continue to hold if the underlying model changes.  On the other hand, a more recent body of work has established lower bounds that are classifier agnostic  (see \cite{Bhagoji2019LowerBO, Pydi2021AdversarialRV, Trillos2020AdversarialCN, bungert2022gamma} for the binary classification case and \cite{trillos2023multimarginal, dai2023characterizing} for the multiclass case).  The classifier agnostic lower bounds are obtained by relaxing the training problem to allow the learner to select any measurable probabilistic classifier (note that in the modern era of neural networks with billions of parameters, this relaxation may be relatively tight). As a result, the lower bounds are in fact universal and have no dependence on the learning model.     Nonetheless, despite this fertile body of work, the actual computation of these bounds along with implementable algorithms has been left largely unexplored.

	The main focus of this paper is to provide efficient algorithms for computing classifier agnostic lower bounds on the optimal adversarial risk.  To do so, we utilize an equivalence discovered in \cite{trillos2023multimarginal} between the relaxed adversarial training problem and a \emph{multimarginal optimal transport} (MOT) problem related to finding barycenters using the $\infty$-Wasserstein distance. Thanks to this connection, we can leverage tools from computational optimal transport to develop efficient algorithms for solving the equivalent MOT barycenter problem.

	%
	
	In general, MOT problems (including the Wasserstein barycenter problem) are NP-Hard \cite{Altschuler2022}. However, we will show that when $\varepsilon$ is not too large (the relevant regime for AT)  it is possible to efficiently solve the particular MOT problem related to AT.    The key insight is that for small values of $\varepsilon$, the search space for the problem can be significantly truncated, allowing for very efficient computation.  Notably, this truncation can only decrease the value of the problem, preserving the guarantee that computed values are truly lower bounds on the optimal adversarial risk.  Equally important is the fact that this truncation is compatible with both linear programming and the Sinkhorn algorithm, two of the most popular methods for solving MOT problems.  Leveraging these two approaches, we propose two very efficient algorithms for solving this problem.  We then validate our algorithms
	with experiments on MNIST and CIFAR-10 to demonstrate the tractability of our approach.

	The closest paper to this work is \cite{dai2023characterizing}, where independently to \cite{trillos2023multimarginal}, the authors relate the multiclass relaxed adversarial learning problem to an equivalent combinatorial optimization problem using the notion of a conflict hypergraph.  Although the work \cite{dai2023characterizing} also explores the idea of truncation, the authors do not provide any tailored algorithms to solve the problem. In contrast, we present two concrete algorithms and provide their computational complexity. 
	
	
	\subsection{Our contributions}
	Our main contributions in this paper are the following.
	\begin{itemize}
		
		\item We introduce and analyze a new algorithm for approximating adversarial attacks based on a stratified and multi-marginal form of Sinkhorn's algorithm. In addition we provide a publicly available implementation of our algorithm.\footnote{Code can be found at \url{https://github.com/MattWerenski/Adversarial-OT}}
		
		\item We give rigorous bounds on the computational complexities of our algorithms based on the user chosen truncation rate. We show that a certain class of MOT problems can be solved very efficiently, despite the fact that MOT problems are in general NP-Hard. 
		
		\item We implement, discuss, and compare against an exact solver based on Linear Programming as was done in \cite{dai2023characterizing}. We also implement fast constructions of \v{C}ech and Rips complexes which may be of broader use, particularly for topological data analysis.
		
	\end{itemize}




	\subsection{Outline}
	The rest of the paper is organized as follows. In section \ref{section : Preliminaries}, we will formally introduce mathematical backgrounds for the adversarial training problem and notation used throughout the rest of the paper.  In section \ref{section : main results}, we will present our two main algorithms along with informal explanations for their efficiency and run time. In section \ref{Sec:MainResults} we present our main theoretical results on the analysis of our algorithm based on Sinkhorn iterations. Rigorous proofs of our main theoretical results will be delayed until the appendix. In section \ref{section : empirical results}, we provide and discuss empirical results obtained from running our proposed algorithms. We will finish off the paper in section \ref{section : conclusions and future works}, with some conclusions and discussions of future directions. Finally, all technical details and proofs will be discussed in the \nameref{section : Appendix}.

	\section{Preliminaries}\label{section : Preliminaries}
	\subsection{Basic concepts and notation}
	Let $(\mathcal{X},d)= (\mathbb{R}^p, || \cdot ||)$ denote the feature space, and let $\mathcal{Y}:=\{1, \dots, K\}$ be the set of $K$ classes for a given classification problem of interest. Let $S_K := \{A \subseteq \mathcal{Y} : A \neq \emptyset \}$ and $S_K(i) := \{A \in S_K \,\colon\, i \in A\}$. Let $\mathcal{Z}:=\mathcal{X} \times \mathcal{Y}$ denote the set of feature-class pairs. Let $\mu \in \mathcal{P}(\Z)$ be a Borel probability measure that represents the ground-truth data distribution. For convenience, we will often describe the measure $\mu$ in terms of its class weights $\mu=(\mu_1, \dots, \mu_K)$, where each $\mu_i$ is a positive Borel measure (not necessarily a probability measure) over $\mathcal{X}$ defined according to:
	\[  
	\mu_i(E) = \mu (E \times \{ i\}),
	\]
	for all Borel measurable $E \subseteq \mathcal{X}$. Notice that the measures $\mu_i$'s are, up to normalization factors, the conditional distributions of features given the specific labels, and $\sum_{i \in \mathcal{Y}} \mu_i(\mathcal{X}) = ||\mu|| = 1$.

	The typical goal of (deterministic) multiclass classification is to find a Borel measurable map $f: \mathcal{X} \rightarrow \mathcal{Y}$ within a certain class (a.k.a. hypothesis class) which minimizes $\mathbb{E}[\ell(f(\mathcal{X}),Y)]$, where $(\mathcal{X},Y) \sim \mu$, and $\ell:\Y \times \Y \rightarrow \mathbb{R}$ is a loss function.  Due to the non-convexity of the space of multiclass deterministic classifiers,  often one needs to relax the space and instead consider \textit{probabilistic} multiclass classifiers $f: \mathcal{X} \rightarrow \Delta_{\mathcal{Y}}$ where 
	$$ \Delta_{\mathcal{Y}} := \left\{ (u_i)_{i \in \mathcal{Y}} : 0 \leq u_i \leq 1, \sum_{i \in \mathcal{Y}} u_i =1 \right\}$$ is the probability simplex over the label space $\mathcal{Y}$.  In what follows, we denote the set of \textit{all} such Borel maps by $\mathcal{F}:= \left\{ f: \mathcal{X} \to \Delta_{\mathcal{Y}} : f \text{ is Borel measurable} \right\}$. Given $f \in \mathcal{F}$ and $x \in \mathcal{X}$ we use $f_i(x)$ to denote it's $i$th component at the point $x$. The value of $f_i(x)$ should be interpreted as the estimated probability that $x$ is in the $i$th class under the classification rule $f$. This extension to probabilistic classifiers is typically unavoidable, especially for adversarial problems in multiclass classification \cite{trillos2023multimarginal, existence_AT}. The objective in this setting is given by
	\begin{equation}\label{eq:UnrobustRiskMinimization}
	\inf_{f \in \mathcal{F}} R(f, \mu) :=  \mathbb{E}_{(X,Y) \sim \mu} [\ell (f(X), Y)) ].
	\end{equation}
	 \eqref{eq:UnrobustRiskMinimization} represents the standard (agnostic) multiclass Bayes learning problem. Through this paper, the loss function $\ell: \Delta_{\mathcal{Y}} \times \mathcal{Y}  \rightarrow \R$ is set to be $\ell(u, i) := 1 - u_i$ for $(u,i) \in \Delta_{\mathcal{Y}} \times\mathcal{Y}$, which is usually referred to as the $0$-$1$ loss. Under the $0$-$1$ loss function the risk $R(f, \mu)$ can be rewritten as 
	\[ 
	R(f, \mu)= \sum_{i \in {\mathcal{Y}}} \int_{\mathcal{X}} \left(1 - f_i(x) \right) d\mu_i(x),
	\]
	and is well known that its solution is the so called Bayes classifier, which admits an explicit form in terms of the distribution $\mu$.	

	\subsection{Adversarial training}

	One popular approach used in adversarial training is \emph{distributionally robust optimization} (DRO). Here the training model is based on a \emph{distribution-perturbing adversary} which can be formulated through the min-max optimization problem
	\begin{equation}\label{def:DRO_model}
	R_{DRO}^*:=\inf_{f \in \mathcal{F}} \sup_{\widetilde{\mu} \in \mathcal{P}(\mathcal{Z})} \left\{ R(f, \widetilde{\mu}) - C(\mu, \widetilde{\mu}) \right\}.
	\end{equation}
	Here, the adversary has the power to select a new data distribution $\widetilde{\mu}$, but they must pay a cost to transform $\mu$ into $\widetilde{\mu}$ given by $C : \mathcal{P}(\mathcal{Z}) \times \mathcal{P}(\mathcal{Z}) \to [0, \infty ],$ which measures how different $\widetilde{\mu}$ is from $\mu$. This forces the learner to select a classifier $f$ that is robust to perturbations of the ground truth data $\mu$ within a certain distance determined by the properties of the chosen cost $C$.
	
	In this work, we consider the family of costs $C_\eps, \eps > 0$ which are transportation costs from $\mu$ to $\widetilde{\mu}$ given by
	\begin{equation*}
	C_\eps(\mu, \widetilde{\mu}) := \sum_{i \in \Y} \inf_{\pi_i \in \Pi(\mu_i, \widetilde{\mu}_i)} \int c_\eps(x,\tilde{x}) d\pi_i(x, \tilde{x}),
	\end{equation*}
	where $\widetilde{\mu}_i$ is defined analogously to $\mu_i$, $\Pi(\mu_i, \widetilde{\mu}_i)$ is the set of probability measures over $\X \times \X$ whose first and second marginals are $\mu_i$ and $\widetilde{\mu}_i$ respectively, and $c_\varepsilon$ is given by
	\begin{equation} \label{assumption: c_eps}
	c_\eps(x,\tilde{x}) = \begin{cases}
	0 & d(x,\tilde{x}) \leq \eps \\
	+\infty & \text{otherwise}
	\end{cases} 
	\end{equation}
	for some distance $d$ on $\X$ and some adversarial budget $\eps$. We will slightly abuse notation and write $C_\eps(\mu_i,\widetilde{\mu}_i)$ to mean the transport cost between $\mu_i$ and $\widetilde{\mu}_i$ with cost $c_\eps$. If $\Pi(\mu_i,\widetilde{\mu}_i)$ is empty, which is the case when $\|\mu_i\| \neq \|\widetilde{\mu}_i\|$, then we take $C_\eps(\mu_i,\widetilde{\mu}_i) = \infty$.
	
	An equivalent perspective is that $\widetilde{\mu}$ is a feasible attack for the adversary only if for all $i \in \mathcal{Y}$ it holds that $W_\infty(\mu_i, \widetilde{\mu}_i) \leq \varepsilon$, where $W_\infty$ denotes the $\infty$-OT distance between measures. In other words, the adversary is only allowed to move each individual data point in the distribution by a distance $\varepsilon$ in the feature space $\cX$.  This shows how the choice of budget $\eps$ is related to the strength of the adversary. As $\eps$ increases, the adversary can make stronger and stronger attacks. 
	
	Since the min in problem \ref{def:DRO_model} is over all possible measurable probabilistic classifiers, the value of problem \eqref{def:DRO_model} provides a universal lower bound for learning problems over \textit{any} family of (Borel measurable) classifiers (e.g., neural networks, kernel machines, etc ) with the same type of robustness enforcing mechanism (i.e., same adversarial cost). For this reason, computing \eqref{def:DRO_model} is of relevance, and in particular an estimated value for \eqref{def:DRO_model} can be used as benchmark when training structured classifiers in practical settings, as has been discussed in the introduction.

	\subsection{An equivalent MOT problem}

	\begin{figure}
		\centering
		\begin{subfigure}[b]{.5\textwidth}
			\centering
			\includegraphics[width=.9\linewidth]{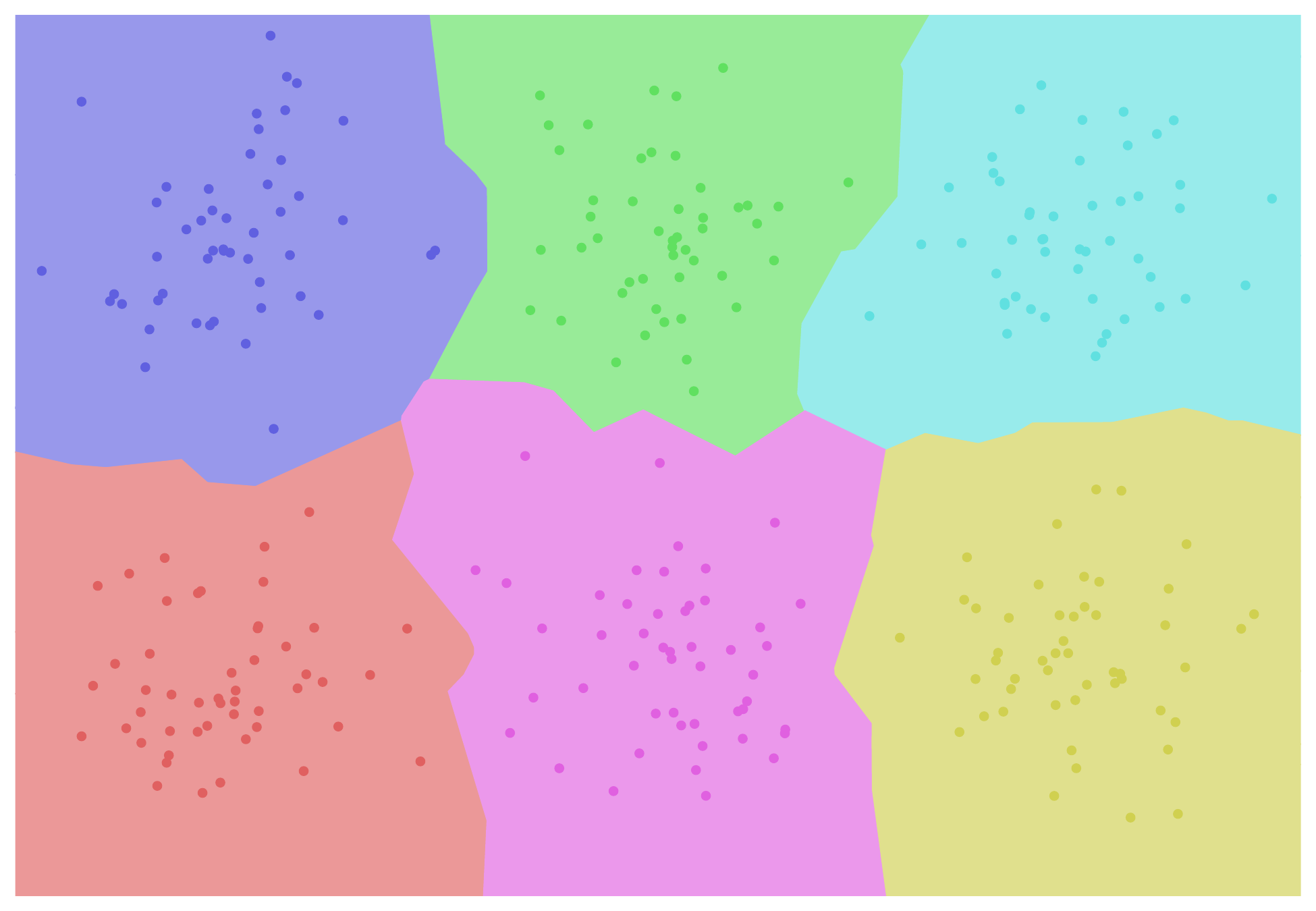}
			\label{fig:sub1}
		\end{subfigure}
        \hfill
		\begin{subfigure}[b]{.5\textwidth}
			\centering
			\includegraphics[width=.9\linewidth]{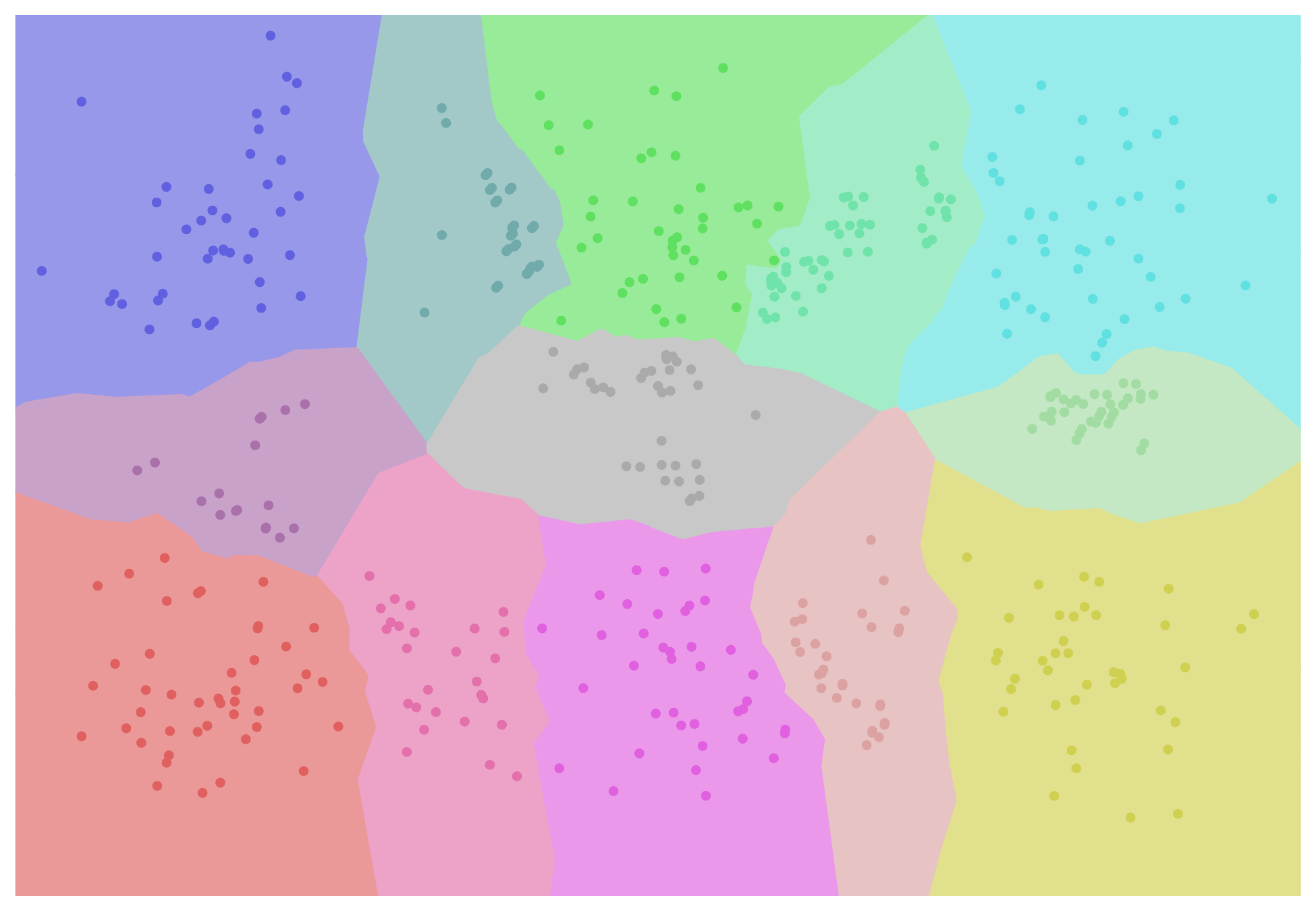}
			\label{fig:sub2}
		\end{subfigure}
		\caption{\textbf{(Top)} A simple six class dataset with 50 points in each class. Filled regions are colored according to the class of the nearest point. \textbf{(Bottom)} The optimal adversarial attack applied to the dataset on the left. The shared colors from the left figure represent the singleton classes $\mu_{\{1\}},...,\mu_{\{6\}}$ and the blended colors represent the various $\mu_{A}$ for $|A| \geq 2$. Points are colored according to the combination of classes they are associated with. Filled regions are colored according to the combination of the nearest point.}
		\label{fig:test}
	\end{figure}

	In \cite{trillos2023multimarginal}, the authors showed that the DRO training problem \eqref{def:DRO_model} is equivalent to an MOT problem related to solving a generalized version of the Wasserstein barycenter problem. In what follows we provide some discussion on this equivalence. Given $\mu=(\mu_1, \dots, \mu_K)$, the \textit{generalized barycenter problem} associated to the cost $c_\veps$ in \eqref{assumption: c_eps}  is
	\begin{equation}\label{def : generalized barycenter problem}
	\min_{\lambda, \widetilde{\mu}_1, \dots, \widetilde{\mu}_K  } \left\{ \lambda(\mathcal{X}) + \sum_{i \in \mathcal{Y}} C_\eps(\mu_i, \widetilde{\mu}_i) : \lambda \succeq \widetilde{\mu}_i \text{ for all } i \in \mathcal{Y} \right\},
	\end{equation}
	where by $\lambda \succeq \widetilde{\mu}_i$ we mean that the positive measure $\lambda$ dominates the measure $\mu_i$. That is, for any non-negative measurable function $g$, it holds that $\int_{\mathcal{X}} g(x) d\lambda(x) \geq \int_{\mathcal{X}} g(x) d\widetilde{\mu}_i(x)$. 
	
	For any feasible collection $\lambda, \widetilde{\mu}_1,...,\widetilde{\mu}_K$ it is possible to perform the following decomposition
	\begin{equation} \label{eq : decomposition}
	\lambda = \sum_{A \in S_k} \lambda_A, \hspace{0.5cm} \lambda_A = \widetilde{\mu}_{i,A} \ \forall \ i \in A, \hspace{0.5cm} \widetilde{\mu}_i = \sum_{A \in S_K(i)} \widetilde{\mu}_{i,A} \ \forall \ i=1,...,K,
	\end{equation}
	where $\lambda_A$ is a measure which accounts for the jointly overlapping mass of the $\widetilde{\mu}_i$ with $i \in A$. See Figure \ref{fig:test} \eqref{def : generalized barycenter problem} and \eqref{eq : decomposition}. From the decomposition of $\widetilde{\mu}_i$ one can also decompose $\mu_i$ into $\mu_i = \sum_{A \in S_K(i)} \mu_{i,A}$ in such a way that $C_\eps(\mu_i, \widetilde{\mu}_i) = \sum_{A \in S_K(i)} C_\eps(\mu_{i,A}, \widetilde{\mu}_{i,A})$. Using the decomposition \eqref{eq : decomposition} and the identity for the cost one can convert \eqref{def : generalized barycenter problem} into the equivalent problem
	\begin{equation}\label{eq : partition of barycenter}
	\min_{(\lambda_A, \mu_{i, A}) \in F} \left\{ \sum_{A \in S_K} \left( \lambda_A(\mathcal{X}) + \sum_{i \in A} C_\eps(\lambda_A, \mu_{i, A}) \right) \right\},
	\end{equation}
	where $F$ is the feasible set defined by
	\begin{equation*}
	F := \left \{ (\lambda_A, \mu_{i,A}) \ : \  \sum_{A \in S_K(i)} \mu_{i,A} = \mu_i \ \forall \ i \in \Y  \right \}.
	\end{equation*}
	
	One can rigorously prove that an optimal  $\lambda_A$ is a barycenter for the set of measures, $\{\mu_{i,A}: i\in A\}$. As in classical (Wasserstein) barycenter problems, \cite{MR2110613, agueh_carlier, MR2801182}, \eqref{eq : partition of barycenter} has an equivalent \emph{stratified MOT} formulation: letting $\mathcal{X}^A := \prod_{i \in A}\mathcal{X}$ and $x_A := (x_i : i \in A) \in \mathcal{X}^A$,
	\begin{equation}\label{eq : stratified MOT}
	\begin{aligned}
	&\min_{\{\pi_A\}_{A \in S_K}} \sum_{A \in S_K} \int_{\mathcal{X}^A} \left ( 1 + c_{\eps,A}(x_A)  \right ) d\pi_A(x_A) \\
	&\text{s.t. } \sum_{A \in S_K(i)} {\mathcal{P}_i}_{\#} \pi_A = \mu_i \hspace{0.5cm} \forall \ i \in \mathcal{Y},
	\end{aligned}
	\end{equation}
	where $\mathcal{P}_i$ is the projection map $(x_A) \mapsto x_i$ onto the $i$-th component for $i \in A$, and for each $A \in S_K$, $c_A$ is defined as 
	\begin{equation}\label{eq : assumption of c_A}
	c_{\eps, A}(x_A) := \inf_{x' \in \mathcal{X}} \sum_{i \in A} c_\eps(x', x_i)
	\end{equation} 
	with the convention that $c_{ \{i\}}=0$ for all $i\in \Y$. It is proved in \cite{trillos2023multimarginal} that 
	\begin{equation*}
	\eqref{def:DRO_model} = 1 - \eqref{def : generalized barycenter problem} = 1 - \eqref{eq : partition of barycenter} = 1 - \eqref{eq : stratified MOT}.
	\end{equation*}
	This paper is devoted to developing and understanding algorithms to solve \eqref{def:DRO_model}, or equivalently to solve \eqref{eq : partition of barycenter} and \eqref{eq : stratified MOT}, respectively.


	\section{Algorithms} \label{section : main results}
	The main contribution of this paper is to suggest two numerical schemes for solving \eqref{def:DRO_model} when the measure $\mu$ is an empirical measure associated to a finite data set. From now on, we assume that each $\mu_i$ is a measure supported over a finite set
	\[
	\mathcal{X}_i : =\spt(\mu_i) \subset \mathcal{X}, \quad |\mathcal{X}_i | = n_i = O(n) \text{ for all } i \in \mathcal{Y}.
	\]
	Also, we use $\mathcal{X}^A := \prod_{i \in A} \mathcal{X}_i$ and $x_A :=(x_i : i \in A) \in \mathcal{X}^A$ as before. The key ideas that make these methods feasible even when $K$ is big is to \emph{truncate interactions} and leverage the fact that the set of feasible interactions is small when the adversarial budget $\varepsilon$ is small. The truncation of \eqref{eq : partition of barycenter} to interactions of level $L \leq K$ is given by
	\begin{equation}\label{eq : truncation of barycenter}
	\min_{(\lambda_A, \mu_{i, A}) \in F_L} \left\{ \sum_{A \in S_K} \left( \lambda_A(\mathcal{X}) + \sum_{i \in A} C_\eps(\lambda_A, \mu_{i, A}) \right) \right\}. 
	\end{equation}
	where 
	\begin{equation*}
	F_L := \left \{ (\lambda_A, \mu_{i,A}) \in F :  |A| \leq L \right \}.
	\end{equation*}
In words, we set $\mu_{i,A} = 0$ (hence, $\lambda_A =0$) for all $A$ such that $|A| > L$.	Similarly, the truncation of \eqref{eq : stratified MOT} to interactions of level $L$ is given by 
	\begin{equation}\label{eq : truncation of stratified MOT}
	\begin{aligned}
	&\min_{\{\pi_A\}_{A \in S^L_K}} \sum_{A \in S^L_K} \int_{\mathcal{X}^A} \left ( 1 + c_{\eps,A}(x_A)  \right ) d\pi_A(x_A) \\
	&\text{s.t. } \sum_{A \in S^L_K(i)} {\mathcal{P}_i}_{\#} \pi_A = \mu_i \text{ for all } i \in \mathcal{Y},
	\end{aligned}
	\end{equation}
where $S^L_K :=\left\{ A \in S_K : |A| \leq L \right\}$ and $S^L_K(i) :=\left\{ A \in S_K(i) : |A| \leq L \right\}$. In words, we set $\pi_A = 0$ for all $A$ such that $|A| > L$.

The truncations of these problems satisfy the following approximation guarantees.
	\begin{proposition} \label{prop : approximation bounds}
		Let $\{ \lambda^*_A \}$ be the optimal measures in \eqref{eq : partition of barycenter}. For $1\leq L \leq K$ we have
		\begin{equation*}
		\eqref{eq : truncation of barycenter} - \eqref{eq : partition of barycenter} \leq \sum_{k > L}^K \left ( \frac{k}{L} - 1 \right ) \sum_{|A| = k} \|\lambda^*_A\|.
		\end{equation*}
		Let $\{\pi_A^* \}$ be the optimal measures in $\eqref{eq : stratified MOT}$. For  $1 \leq L \leq K$ we have
		\begin{equation*}
		\eqref{eq : truncation of stratified MOT} - \eqref{eq : stratified MOT} \leq \sum_{k > L}^K \frac{k}{L} \sum_{|A| = k} \|\pi_A^*\|.
		\end{equation*}
	\end{proposition}
	The proofs of both facts, presented in Appendix \ref{App:Prop1}, are given by constructing measures $\{\lambda^L_A : |A| \leq L\}$ and $\{\pi^L_A : |A| \leq L \}$ from the optimal measures $\{ \lambda^*_A \}$ and $\{ \pi^*_A \}$, respectively, which are feasible for \eqref{eq : truncation of barycenter} and \eqref{eq : truncation of stratified MOT} and obtain the bounds above.

    \begin{figure}
        \centering
        \includegraphics[width=0.7\linewidth]{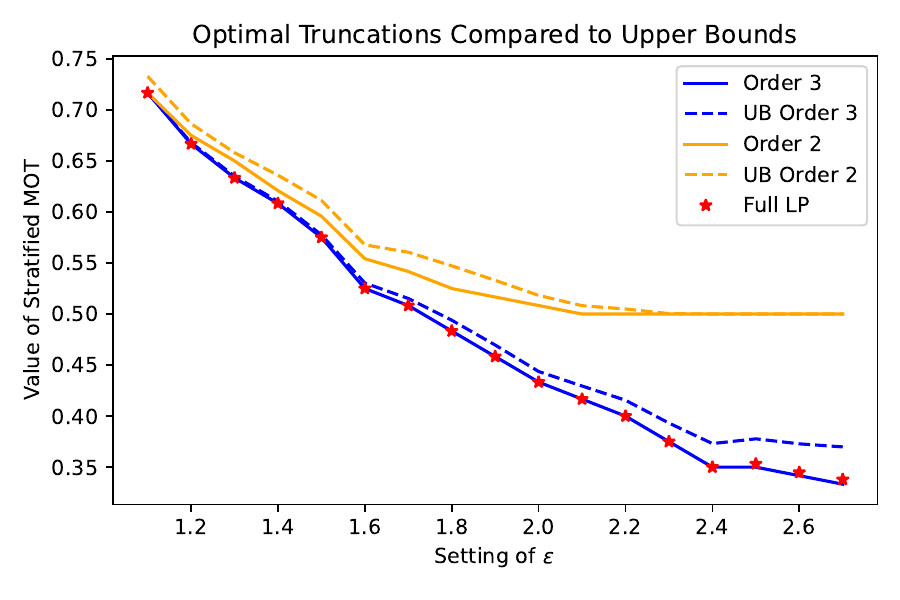}
        \caption{Plots of the value of \eqref{eq : truncation of stratified MOT} and the upper bound provided by Proposition \ref{prop : approximation bounds} for a range of settings of $\varepsilon$ and $K = 2,3$ as well as the untruncated values. These are derived from synthetic data using 20 samples from six classes.} 
        \label{fig : upper bounds}
    \end{figure}
     
	An important takeaway from Proposition \ref{prop : approximation bounds} is that if the optimal measures do not heavily utilize interactions beyond the truncation level $L$, then one can faithfully recreate the attack without leveraging these interactions at all. Empirically, we observe that this is often the case: see experiments in section \ref{section : empirical results}. We also illustrate these bounds on synthetic data in Figure \ref{fig : upper bounds} where the upper bounds are $\eqref{eq : stratified MOT}$ plus the right hand side in Proposition \ref{prop : approximation bounds}. We observe that these bounds in this example are close to the actual values of \eqref{eq : truncation of stratified MOT}.  Importantly, notice that our truncation procedure does not reduce the total number of classes $K$, indeed, all classes $\{1, \ldots, K\}$ continue to influence adversarial attacks.

 \subsection{Feasible labeled interactions}
In addition to truncation, the other key step in our method is to reduce the computational complexity by restricting the search space to feasible interactions only.  Note that this does not change the problem whatsoever, as the adversary cannot combine points that do not lie in a common $\varepsilon$ ball (or more generally points whose joint transport cost is infinite) whenever $\varepsilon$ is small. We make this concept rigorous in the following definition. 

	\begin{definition}
		A set of points $\{(x_i, y_i)\}_{i=1}^k \subset \spt(\mu)$ is a \textit{valid interaction} with $1 \leq k \leq K$ if the labels $y_i$ are all distinct and 
		\begin{equation}
		\inf_{x' \in \X} \sum_{i=1}^k c_{\eps}(x',x_i) < \infty.
		\end{equation}
		The set of all valid interactions will be denoted by $\mathcal{I}_K$ and the set of all valid interactions of size at most $L$ will be denoted by $\mathcal{I}_L$.
	\end{definition}
	Note that this definition allows valid interaction sets to be of size $1$ to $K$ and in fact every singleton set is a valid interaction by choosing $x' = x_i$.

To leverage the computational gains of our approach, we must compute the feasible ordered interactions.  	Here we present an iterative method for computing these interactions, starting from singleton interactions (namely, all points of each class) and then building up to interactions of order $L\leq K$ where $L$ is the user-chosen truncation level.  To facilitate the description of our iterative method it will be useful to introduce for each $A\in S_K$ the feasible sets
\begin{equation}
    F_{A}:=\{ \{(x_i, y_i)\}_{i=1}^{|A|}: \inf_{x' \in \X} \sum_{i=1}^{|A|} c_{\eps}(x',x_i) < \infty, \; \bigcup_{i=1}^{|A|} \{y_i\}=A\}.
\end{equation}
Note that $F_A$ is the set of all feasible interactions with labels corresponding to the set $A$.
Now we are ready to present Algorithm \ref{alg : contruct interactions} for computing feasible interactions.   
	
	\begin{algorithm}
		\begin{algorithmic}
			\caption{Construct $\mathcal{I}_L$}\label{alg : contruct interactions}
			\Require{$X$ : data set, $\varepsilon$ : adversarial budget, $L$: truncation level}
			\State{For each $i \in \mathcal{Y}$, set $F_{\{i\}} = \{(i,1),...,(i,n_i)\}$.}
			\For{$k=2, \dots, L$}
			\For{$A,A' \in S^L_K$ with $|A| = |A'| = k-1, |A \cap A'| = k - 2$ and $F_A, F_{A'} \neq \varnothing $}
			\For{Each $C \in F_A, C' \in F_{A'}$ with $|C \cap C'| = k-2$}
			\State{Check if there exists a  point $x$ within $\varepsilon$ of every point in $C \cup C'$}.
			\State{If so, add $C \cup C'$ to the set $F_{A \cup A'}$.}
			\EndFor
			\EndFor
			\EndFor
			\Ensure{$\mathcal{I}_L := \bigcup_{A \in S^L_K} F_A$.}
		\end{algorithmic}
	\end{algorithm}
	
	The main difficulty in implementing Algorithm \ref{alg : contruct interactions} is to ensure that the checks for $|A \cap A'| = k-2$ and $|C \cap C'| = k-2$ are efficient and are not done by enumerating all possibilities. With a proper implementation, the most time consuming step is checking when a point $x$ is within $\varepsilon$ of every point in $C \cup C'$. This is often a non-trivial geometric problem. For example in $\mathbb{R}^n$ with the Euclidean distance it requires checking if as many as $L$ spheres in $\mathbb{R}^n$ of radius $\varepsilon$ have a mutual intersection. One geometry where this calculation is particularly simple is when using $d(x,y) = \norm{x-y}_\infty$, where the problem is reduced to finding the intersection of axis-aligned rectangles.

	In general the speed of Algorithm \ref{alg : contruct interactions} can be estimated by  $O(L|C(\varepsilon)|m(L))$ where $m(L)$ is the computational complexity required to check the existence of a point $x$ for groups of size at most $L$ (this is polynomial in $L$ and the dimension of the space $d$). The overall complexity is essentially at worst the same as trying every possible group of $L$ or fewer points, which is what may be required if $\varepsilon$ is large enough that a majority of all the groups of size $L$ are feasible. However, in practice, there are often far fewer higher-order interactions, which leads to a much faster algorithm. 
	
	\subsection{Linear Programming Approach}

	The first method for solving problem (\ref{eq : stratified MOT}) or its truncated version (\ref{eq : truncation of stratified MOT}) that we discuss in this paper is based on \emph{linear programming} (LP).  The key object in the LP approach is the  \textit{interaction matrix} denoted by $I \in \{0,1\}^{\spt (\mu) \times \mathcal{I}_K}$. The rows of this matrix are indexed by points $z = (x,y) \in \spt (\mu)$ while the columns are indexed by the valid interactions $\iota \in \mathcal{I}_K.$ For $z \in \spt (\mu)$ and $\iota \in \mathcal{I}_K$ the corresponding entry in $I$ is given by
	\begin{equation*}
	I[z,\iota] = \begin{cases}
	1 & z \in \iota \\
	0 & \text{otherwise}.
	\end{cases}
	\end{equation*}
	We will also need to define a \textit{marginal vector} $m \in [0,1]^{\spt (\mu)}$ with $m[z] = \mu[\{z\}]$. With these definitions, we can formulate a linear program which solves \eqref{eq : partition of barycenter}:
	\begin{equation} \label{eq : linear program}
	\begin{aligned}
	\min_{w \in [0,1]^{\mathcal{I}_K}} &\hspace{1cm} \sum_{\iota \in \mathcal{I}_K} w[\iota]  \\
	\text{s.t.} & \hspace{1cm} Iw = m.
	\end{aligned}
	\end{equation}
	In an analogous way to the above, one can truncate the problem with the truncation level $L$, and obtain the truncated LP 
	\begin{equation}  \label{eq : trunaction of linear program}
	\begin{aligned}
	\min_{w \in [0,1]^{\mathcal{I}_L}} &\hspace{1cm} \sum_{\iota \in \mathcal{I}_L} w[\iota]  \\
	\text{s.t.} & \hspace{1cm} Iw = m.
	\end{aligned}
	\end{equation}
	\begin{proposition}
		The optimization problems \eqref{eq : stratified MOT} and \eqref{eq : linear program} are equivalent. The optimization problems \eqref{eq : truncation of stratified MOT} and \eqref{eq : trunaction of linear program} are equivalent. As a result the truncated LP obtains the same approximation as in Proposition \ref{prop : approximation bounds}.
	\end{proposition}
	\begin{proof}
		We will show how to convert between the two problems. We will only do this for the untruncated versions as the truncated versions are done in an identical fashion.
		
		First let $\{\pi_A\}$ be feasible and have finite cost for \eqref{eq : stratified MOT}. To each point  $(x_1,...,x_K) \in \spt (\pi_A)$ we can assign a set $\iota = \{(x_i, i) \ | \ i \in A \}$. Clearly the labels in $\iota$ are unique. In addition, since we assume that the $\{\pi_A\}$ achieve a finite cost we must have $c_A(x_1,...,x_K) = \inf_{x' \in \X} \sum_{i\in A} c(x',x_i) < \infty$ which shows that $\iota$ is a valid interaction. Set
		\begin{equation}
		w[\iota] = \pi_A\left [ \left \{ (x_1',..,x_K') \ | \ x_i' = x_i \ \forall \ i \in A  \right \} \right ].
		\end{equation}
		The projection sum constraint ensures that $Iw = m$.
		
		Now consider a feasible $w$. For a valid interaction $\iota = \{(x_i,y_i)\}$ let $A = \{y_i\}$. For $i \notin A$ let $x_i$ be an arbitrary point in the support of $\mu_i$. Now set
		\begin{equation}
		\pi_A\left [ \{(x_1,...,x_K)\} \right ] = w[\iota].
		\end{equation}
		The summation constraint in the LP ensures that the $\pi_A$ are feasible.
	\end{proof}
	The gain in \eqref{eq : linear program} is that the optimization occurs over a space of dimension  $|\mathcal{I}_K|$ (which we expect to be small when $\epsilon$ is small), while the dimension of the problem in $\eqref{eq : stratified MOT}$ is $(2^K-1) n^K$ when $\spt (\mu)_i = O(n)$ for every $i$ (since there are $2^K-1$ sets $A \in S_K$ and each $\pi_A$ is of size $n^K$). In the worst case setting, it may be that $|\mathcal{I}_K| = (2^K-1) n^K$, which happens for example when $ \bigcup_{i=1}^K \spt(\mu_i) \subset \overline{B}(0,\varepsilon)$, a closed ball with radius $\varepsilon$. However, in typical settings, one should expect $|\mathcal{I}_K|$ to be much smaller.
	
	Truncating \eqref{eq : linear program} down to \eqref{eq : trunaction of linear program} allows the problem to be solved even more quickly since we eliminate interactions of order larger than $L$.  Note that  in the worst case $|\mathcal{I}_L|$ has size at most $(2^L - 1)n^L$, but again we expect this to be much smaller when $\varepsilon$ is not too large.

    Once one has obtained an optimizer $w^*$ of \eqref{eq : linear program} one can easily compute an optimal adversarial attack $\widetilde{\mu}=(\widetilde{\mu}_1, \dots, \widetilde{\mu}_K)$ and a corresponding optimal generalized barycenter $\lambda$. Let $w^*$ be an optimizer for \eqref{eq : linear program}. An optimal generalized barycenter can be recovered via
	\begin{equation*}
	\lambda = \sum_{C \in \mathcal{I}_K} w^*(C) \delta_{b(C)}
	\end{equation*}
	where $b(C)$ returns any point $x$ such that $\{x^i_{l_i} : (i,l_i) \in C\} \subset \overline{B}(x,\varepsilon)$, furthermore, an optimal adversarial attack $\{ \widetilde{\mu}_1, \dots, \widetilde{\mu}_K \}$ can be recovered via
	\begin{equation*}
	\widetilde{\mu}_i = \sum_{A \in S_K(i)} \sum_{C \in F_A} w^*(C)\delta_{b(C)}.
	\end{equation*}
     The total mass of the measures is correctly preserved because of the constraint $I w = [\mu_1, \dots, \mu_K]^T$. From the preceding equations, it is clear that $\lambda$ dominates $\widetilde{\mu}_i$ for each $i \in \mathcal{Y}$. In addition, it is also easy to recover the transformation $\mu_i \mapsto \widetilde{\mu}_i$.
 The above analysis implies that $\eqref{def : generalized barycenter problem} = \eqref{eq : linear program}$, hence, the optimal adversarial risk is obtained by
	\begin{equation*}
 1 - \eqref{eq : linear program} = 1 - \sum_{C \in \mathcal{I}} w^*(C).
	\end{equation*}
Given an optimizer $w^*_L$ to the level $L$ truncated problem, one can analogously compute all of the corresponding truncated quantities (barycenter, adversarial attacks, and risk) by replacing $w^*, \mathcal{I}, S_K(i)$ by $w^*_L, \mathcal{I}_L, S_L(i)$ respectively in the above formulas.

	\subsubsection{Complexity Considerations of LP approach}
	
	In general, the optimization problem involves a vector $w$ whose length is determined by $|\mathcal{I}_L|$ as well as a sparse matrix $I$ with at most $L|\mathcal{I}_L|$ non-zero entries (although this is quite pessimistic) for some choice of truncation level $L\leq K$. It is therefore essential to control $|\mathcal{I}_L|$. A straightforward calculation can show that 
	\begin{align*}
	\mathbb{E}||\mathcal{I}_L|| = \sum_{A \in S_K, |A|\leq L} \left [ \prod_{i \in A} n_k \right ] \mathbb{P} \left \{ \{X_i\}_{i \in A} \subset \overline{B}(x,\varepsilon) \text{ for some } x \right \}
	\end{align*}
	where $X_i \sim \mu_i$ are independent random variables. It is therefore crucial that the classes are in some sense well-separated as this will control the probability of the formation of an $\varepsilon$-interaction. It may be worthwhile to analyze cases where one can cleanly bound the probability on the right hand side. For example, if $X_i \sim N(m_i, \Sigma_i)$, then one may reasonably expect to bound the probability by a function of the values of $\{(m_i, \Sigma_i)\}$.

    \subsection{Entropic Regularization Approach} \label{sec:EntropyRegul}
	
	The second approach that we consider in this paper is based on Sinkhorn iterations, which here are adapted to be able to solve the entropy-regularized truncated version of problem, \eqref{eq : truncation of stratified MOT}. Sinkhorn iterations were originally proposed in \cite{sinkhorn1964relationship, sinkhorn1967concerning}, and in the past decade have been extensively studied in the optimal transport community \cite{cuturi2013sinkhorn, cuturi2014fast, benamou2015iterative, altschuler2017near, lin2022complexity} in a variety of settings. These extensive algorithmic and theoretical developments have been key factors in the increased use of optimal transport in modern machine learning by practitioners.

	
	Let $L \leq K$ be the fixed level of truncation in problem \eqref{eq : truncation of stratified MOT} and recall the notation $S^L_K = \{A \in S_K : |A| \leq L \}$ and $S^L_K(i) = \{A \in S^L_K : i \in A \}$. Throughout the discussion in this section, we will consider a general family of cost tensors $\{ c_A \}_{A}$  that are non-negative (and that may possibly take the value $\infty$) and satisfy $c_{\{  i\}} \equiv 0$ for all $i \in \Y$. The main example to keep in mind is the collection of cost tensors defined as in \eqref{eq : assumption of c_A}, since these are the cost tensors that are connected to the adversarial training problem.

	The $L$-level truncated entropic regularization problem associated to \eqref{eq : truncation of stratified MOT} (adapted in the obvious way to arbitrary cost tensors) is defined as
	\begin{equation}\label{eq : truncated entropic problem}
	\begin{aligned}
	&\min_{\{\pi_A\}_{A \in S^L_K}} \sum_{A \in S^L_K} \sum_{\mathcal{X}^A} \left ( 1 + c_A(x_A)  \right ) \pi_A(x_A) - \eta H(\pi_A) \\
	&\text{s.t. } \sum_{A \in S^L_K(i)} {\mathcal{P}_i}_{\#} \pi_A = \mu_i \hspace{0.5cm} \text{ for all } i \in \mathcal{Y},
	\end{aligned}
	\end{equation}
	where $H(\pi_A):= - \sum_{x_A} \left( \log \pi_A(x_A) - 1 \right) \pi_A(x_A)$. In general, a Sinkhorn-based algorithm for finding solutions to a regularized transport problem over couplings aims at solving the corresponding dual problem by a coordinatewise greedy update. In this case, as discussed in Appendix \ref{app:Dual}, the dual problem associated to \eqref{eq : truncated entropic problem} (here written as a minimization problem for convenience) takes the form:
	\begin{equation}\label{eq:truncated_obj}
	\begin{aligned}
	&\min_{ \{g_i\}_{i \in \Y}} \mathcal{G}^L( \{g_i\} )\\
	&:= \sum_{A \in S^L_K} \sum_{\mathcal{X}^A} \exp \left( \frac{1}{\eta} \left( \sum_{i \in A} g_i(x_i) - \left(1 + c_A(x_A) \right) \right) \right) - \sum_{i \in \mathcal{Y}} \sum_{\mathcal{X}_i} \frac{g_i(x_i)}{\eta} \mu_i(x_i).
	\end{aligned}
	\end{equation}

	
	For a given value of the dual variables $\{ g_i\}_{i \in \Y}$ we define an induced family of couplings $\{\pi_A(g)\}_{A \in S_K^L}$ according to
	\begin{equation}
	\label{eq:PiFromg}
	\pi_A(g)(x_A) := \exp \left(\frac{1}{\eta}\left(\sum_{i \in A} g_i(x_i) - (1 + c_A(x_A)) \right) \right), \quad \forall x_A \in \mathcal{X}^A, \quad \forall A \in S_K^L. 
	\end{equation}
	In general, these couplings do not satisfy the constraints in \eqref{eq : truncated entropic problem}, but if $g=g^*$ is a solution to \eqref{eq:truncated_obj}, then $\{ \pi^*_A \}_{A \in S_K^L}$ is optimal for problem \eqref{eq : truncated entropic problem}. Moreover, it can be shown that a collection of couplings of the form \eqref{eq:PiFromg} that in addition satisfy the constraints in \eqref{eq : truncated entropic problem} is in fact the global solution of \eqref{eq : truncated entropic problem}.


	\begin{remark}
		We want to highlight two new challenges in our setting in relation to other settings that have been studied in the literature of optimal transport, including standard MOTs. The first obstacle is that we must consider the set of coupling tensors of different orders simultaneously, rather than a single coupling tensor. The second obstacle is to account for the imbalance of marginal distributions. Unlike typical MOT problems, each marginal $\mu_i$ has a distinct mass. We address these issues using and expanding on the methods from previous works such as \cite{altschuler2017near, lin2022complexity}. Further details regarding the convergence analysis will be provided in section \ref{Sec:MainResults} and in the Appendix. 
	\end{remark}

	\begin{remark}
		\label{rem:Invariance}
		It is well known that the dual variables to standard MOT problems (i.e. Kantorovich potentials) satisfy a useful invariance.  For standard problems, given dual variables  $(g_1, \dots, g_K)$, the value of the MOT dual problem at $(g_1, \dots, g_K)$ will remain unchanged if one adds a constant vector with mean zero $(h_1, \dots, h_K)$ (i.e. $\sum_{i=1}^K h_i=0$) to $(g_1, \dots, g_K)$  see \cite{vialard2019elementary, di2020optimal, carlier2022linear}. However, our dual problem does not have this property.  This creates an additional layer of difficulty that we will need to overcome, as this invariance is often leveraged in Sinkhorn-type algorithms.
	\end{remark}

 \begin{remark}
 If $c_A(x_A)=+\infty$, then it is clear that the term $x_A$ in the sum over $\mathcal{X}^A$ does not contribute to the value of (\ref{eq:truncated_obj}).  Therefore, the sum over $\mathcal{X}^A$ can be reduced to the sum over $F_A$, which is computed by Algorithm \ref{alg : contruct interactions}.  As in the LP approach, this reduction can significantly reduce the computational complexity of the problem when the adversarial budget is small. 
 \end{remark}


	To solve the entropic regularization problem \eqref{eq:truncated_obj} and as a consequence also solve \eqref{eq : truncated entropic problem}, we introduce Algorithm \ref{alg : multi-sinkhorn} below.  Note in light of the last remark, when computing ${\mathcal{P}_i}_{\#} \pi_A(g^t)$, one needs only to sum over $x_A\in F_A$.

	\begin{algorithm}
		\begin{algorithmic}
			\caption{Truncated Multi-Sinkhorn (without rounding)}\label{alg : multi-sinkhorn}
			\Require{$X$ : data set, $\{c_A\}_A$ : cost tensors, $\mu=(\mu_1, \dots, \mu_K)$ : empirical distribution, $\varepsilon$ : adversarial budget, $L$ : truncation level, $\delta'$: parameter for stopping criterion.}
			\State{\textbf{Initialization.} $t=0$ and $g_i = \boldsymbol{0} \in \mathbb{R}^{n_i}$ for each $i \in \mathcal{Y}$.}
			\While{$E_t > \delta'$}
			\State{\textbf{Step 1.} Choose the greedy coordinate $I := \argmax_{1 \leq i \leq K} D_{\text{KL}}\left(\mu_i || \sum_{A \in S^L_K(i)} {\mathcal{P}_i}_{\#} \pi_A(g^t) \right)$.}
			\State{\textbf{Step 2.} Compute $g^{t+1}=(g_1^{t+1}, \dots, g_K^{t+1})$ by
				\[
				g_i^{t+1} =
				\begin{cases} 
				g_i^{t} + \eta \log \mu_i - \eta \log \left(\sum_{A \in S^L_K(i)}  {\mathcal{P}_i}_{\#} \pi_A(g^t)\right), & \text{ if } i = I\\
				g_i^{t}, & \text{ otherwise }.
				\end{cases}
				\]      
			}
			\State{\textbf{Step 3.} Set new $t$ to be $t+1$.}
			\EndWhile
			\Ensure{$\{\pi_A(g^t) \}_{A \in S^L_K}$.}
		\end{algorithmic}
	\end{algorithm}
	The stopping criterion we use for Algorithm \ref{alg : multi-sinkhorn} is $E_t \leq \delta'$ for some prespecified $\delta'>0$, where $E_t$ is
	\begin{equation}\label{eq : stopping criterion}
	E_t:= \sum_{i \in \mathcal{Y}} || \mu_i -     \sum_{A \in S^L_K(i)} {\mathcal{P}_i}_{\#} \pi_A(g^t) ||_1,
	\end{equation}
	i.e., $E_t$ is the addition of the discrepancies in marginal constraints at iteration $t$.

	Following the analysis presented in section \ref{app:AnalysisAlg4}, one can deduce that the sequence of iterates produced by Algorithm \ref{alg : multi-sinkhorn} induces a collection of couplings $\{ \pi_A(g^t) \}_{A \in S_K^L}$ (via \eqref{eq:PiFromg}) that converge toward the unique solution of \eqref{eq : truncated entropic problem} as $t \to \infty$ . However, when Algorithm \ref{alg : multi-sinkhorn} is stopped at a finite iteration $T$, it is not guaranteed that the current $\{\pi_A(g^T)\}_{A \in S_K^L}$ is feasible for the original problem \eqref{eq:truncated_obj}. In order to obtain feasibility at every iteration, it is important to introduce a rounding scheme. The scheme we use here is an adaptation of one proposed in \cite{altschuler2017near} and later extended by \cite{lin2022complexity} in multimarginal setting. However, our rounding scheme needs to take into account the fact that multiple couplings appear in each individual marginal constraint.

	\begin{algorithm}
		\begin{algorithmic}
			\caption{Rounding}\label{alg : Round}
			\Require{$\{\pi_A\}_{A \in S_K^L}$ : couplings, $\mu=(\mu_1, \dots, \mu_K)$ : empirical distribution.}
			\State{\textbf{Initialization.} $\pi^{(0)}_A = \pi_A$ for all $A \in S^L_K$.}
			\For{$i = 1, \dots, K$}
			\State{Compute $z_i := \min \left\{ \mathds{1}_{n_i}, \mu_i/ \sum_{A \in S^L_K(i)} {\mathcal{P}_i}_{\#} \pi^{(i - 1)}_A \right\} \in \mathbb{R}^{n_i}_+$ .}
			\For{ $x_i \in \mathcal{X}_i$}
			\State{Set for all $A$ and all $x_A$ containing $x_i$ in its coordinate $i$:
				\[
				\pi^{(i)}_A(x_A) =
				\begin{cases} 
				z_i(x_i) \pi^{(i-1)}_A(x_A), & \text{ if } i \in A\\
				\pi^{(i-1)}_A(x_A), & \text{ otherwise.} 
				\end{cases}
				\]      
			}
			\EndFor
			\EndFor
			\State{Compute $\text{err}_i := \mu_i - \sum_{A \in S^L_K(i)} {\mathcal{P}_i}_{\#} \pi^{(K)}_A$ for each $i \in \mathcal{Y}$}
			\State{Compute 
				\[
				\widehat{\pi}_A =
				\begin{cases} 
				\pi^{(K)}_{\{i\}} + \text{err}_i, & \text{ if } A = \{i\}\\
				\pi^{(K)}_A, & \text{ otherwise.} 
				\end{cases}
				\]}
			\Ensure{$\{\widehat{\pi}_A\}_{A \in S_K^L}$.}
		\end{algorithmic}
	\end{algorithm}

	The truncated entropic regularization algorithm (with rounding) is the following.
	\begin{algorithm}[H]
		\begin{algorithmic}
			\caption{Entropic regularization with rounding,}\label{alg : entropic regularization}
			\Require{Fix $\delta > 0$ and $L \leq K$. Set $\eta = \frac{\delta / 2}{  2L\log(C^*Kn) }$ and $\delta' = \frac{\delta / 2}{2L \max_{A \in S^L_K} |1+c_A \mathds{1}_{c_A < \infty} |}.$}
			\State{\textbf{Step 1.} For each $i \in \mathcal{Y}$, set
				\[
				\widetilde{\mu}_i := \left(1 - \frac{\delta'}{4K} \right)\mu_i + \frac{\delta' ||\mu_i||}{4K n_i } \boldsymbol{1}_{n_i}.
				\]
				Let $\widetilde{\mu}:= (\widetilde{\mu}_1, \dots, \widetilde{\mu}_K)$.}
			\State{\textbf{Step 2.} Compute $\{ \widetilde{\pi}_A : A \in S^L_K \}$ by Algorithm \ref{alg : multi-sinkhorn} with $\{c_A\}_A$, $\eta$, $\widetilde{\mu}$ and $\delta'/2$.}
			\State{\textbf{Step 3.} Obtain $\{ \widehat{\pi}_A : A \in S^L_K \}$ by  Algorithm \ref{alg : Round} with $\{ \widetilde{\pi}_A : A \in S^L_K \}$ and $\widetilde{\mu}$.}
			\Ensure{$\{\widehat{\pi}_A\}$.}
		\end{algorithmic}
	\end{algorithm}

    As we will discuss in Section \ref{Sec:MainResults}, Algorithm \ref{alg : entropic regularization} returns a $\delta$-approximate solution to the unregularized truncated problem \eqref{eq : truncation of stratified MOT}.     

	\begin{remark}
		Here, $C^*=\max_{i \in \mathcal{Y}}\frac{n_i}{n}$ appearing in the choice of entropic parameter $\eta$ is a constant independent of all other parameters, $n, K, L$ and $\delta$.

		\textbf{Step 1} of Algorithm \ref{alg : entropic regularization} is necessary unless each $\mu_i$ is dense on $\mathcal{X}_i$. Algorithm \ref{alg : multi-sinkhorn} suffers when updating potentials if there is some $\mu_j$ with a very small mass. This weakness is not specific to this setup but a commonly observed phenomenon in variants of Sinkhorn-type algorithms.

        The choice of $\eta$ is adapted for the theoretical analysis which considers the worst case. In practice, however, a larger choices of $\eta$ works well. Sinkhorn folklore suggests that $0.05$ is a good choice for $\eta$ in most applications.  See Section \ref{ssec:eta_delta} for further discussion of how to choose $\eta$ along with empirical results.

        If one is interested in the lower bound of the adversarial risk only, it is fine to skip \textbf{Step 3}, or Algorithm \ref{alg : Round}. Skipping \textbf{Step 3} has almost no effect on computing the risk. However, the main cost of the algorithm is \textbf{Step 2}, Algorithm \ref{alg : multi-sinkhorn}, and the computational cost of \textbf{Step 3} is minor comparing to that of \textbf{Step 2}.
     \end{remark}

	\subsection{Theoretical Results of Entropic Regularization Approach}
	\label{Sec:MainResults}
	
	In this Section, we state our main theoretical results, where we summarize our analysis of the approach for solving \eqref{eq : truncation of stratified MOT} that is based on the entropic regularization presented in section \ref{sec:EntropyRegul}. Our first main result describes the number of iterations required to achieve the stoppin criteria for Algorithm \ref{alg : multi-sinkhorn}.
	
	\begin{theorem}\label{thm : iteration}
		Let $\{g^t\}_{t \in \mathbb{N}}$ be generated by Algorithm \ref{alg : multi-sinkhorn}. For a sufficiently small fixed $\delta'$, the number of iterations $T$ to achieve the stopping criterion $E_T \leq \delta'$ satisfies
		\begin{equation}\label{eq : number of iteration}
		T \leq  2+ \left \lceil \frac{  \G^L(g^0) }{\min_{i \in \Y} \lVert\mu_i\rVert_1} \right \rceil  + \frac{14 K^2 \overline{R}}{ \eta \delta'},
		\end{equation}
		where
		
		\begin{equation}
		\overline{R} :=  L -  \eta  \log \min_{j \in \mathcal{Y}, y \in \mathcal{X}_j} \mu_j(y) + \eta L\log(K  C^* n).
		\label{def:OvR} 
		\end{equation}
	\end{theorem}

	\nc

	\begin{remark}
		Recall that $g^0 = \boldsymbol{0}$. The initial value of the objective function is bounded above by
		\[
		\G^L(g^0) \leq \sum_{A \in S^L_K} \sum_{\mathcal{X}^A} \exp \left( -\frac{1}{\eta} \right) \leq C^* \exp \left( L \log (Kn) -\frac{1}{\eta} \right).
		\]
		Taking $\eta = O \left( \frac{1}{L \log (Kn)} \right) $ as we will do in Algorithm \ref{alg : entropic regularization}, we see that $\G^L(g^0) = O(1)$. Hence, $ \frac{\G^L(g^0)}{\min_{j \in \Y} || \mu_j ||_1} = O(1)$.
		Moreover, if we assume that $\mu_i(x_i) \sim \frac{1}{n}$ for all $x_i$ and all $i$, the last term in \eqref{eq : number of iteration} is $O(\log^2(n))$. As a consequence, Algorithm \ref{alg : multi-sinkhorn} exhibits almost linear convergence.

	\end{remark}

	\begin{remark}
	Notice that the number of iterations in Theorem \ref{thm : iteration} does not depend on the specific cost tensors $c_A$. Only the non-negativity of the cost tensors and the assumption $c_{\{ i \}\}} \equiv 0$ for all $i \in \Y$ play a role in the estimated number of iterations.
	\end{remark}
	
	To prove Theorem \ref{thm : iteration}, we adapt to our setting the analysis for standard MOT problems presented in \cite{altschuler2017near,lin2022complexity}. In our setting, the marginal distributions need not have the same total mass and each marginal constraint depends on multiple couplings of different orders simultaneously. In addition, the dual potentials in our setting lack an invariance property that is present in the standard setting, which facilitates the analysis in that case (see Remark \ref{rem:Invariance}). As a result, we require a more careful analysis for the decrement of energy at each step of the algorithm. The proof of Theorem \ref{thm : iteration} is presented at the end of Appendix \ref{app:AnalysisAlg4}, after proving a series of preliminary estimates.

	In order to analyze Algorithm \ref{alg : entropic regularization}, we need the following estimates on the output of the rounding scheme.
	
	\begin{theorem}\label{thm : error analysis of round}
		Let $\{\pi_A : A \in S^L_K\}$ be a set of couplings and $\mu= (\mu_1, \dots, \mu_K)$ be a sequence of finite positive vectors. Then Algorithm \ref{alg : Round} returns a set of couplings $\{\widehat{\pi}_A : A \in S^L_K\}$ which satisfies: for all $i \in \mathcal{Y}$
		\[
		\sum_{A \in S^L_K(i)} {\mathcal{P}_i}_{\#} \widehat{\pi}_A = \mu_i,
		\]
		as well as the error bound
		\[  
		\sum_{A \in S^L_K} ||\widehat{\pi}_A - \pi_A ||_1 \leq L \sum_{i \in \mathcal{Y}} ||\sum_{A \in S^L_K(i)} {\mathcal{P}_i}_{\#} \pi_A - \mu_i ||_1.
		\]
	\end{theorem}

%
%

	Finally, we can combine Theorems \ref{thm : iteration} and \ref{thm : error analysis of round} to prove that Algorithm \ref{alg : entropic regularization} outputs a $\delta$-approximate solution for \eqref{eq : truncation of stratified MOT}, the truncated version of \eqref{eq : stratified MOT}.

	\begin{theorem}\label{thm:AnalysisALgoWithRound}
	Algorithm \ref{alg : entropic regularization} returns a $\delta$-approximate optimal solution for \eqref{eq : truncation of stratified MOT}. Moreover, if $\min_{j \in \mathcal{Y}, y \in \mathcal{X}_j} \mu_j(y) = \Omega(n^{-1})$  and $C^*:= \max_{i \in \mathcal{Y}} \frac{n_i}{n} = O(1)$, Algorithm \ref{alg : entropic regularization} requires 		
    \[
		O \left(  \frac{   L^2 K^{2} \max_{A \in S_K^L} (1 + c_A \mathds{1}_{c_A < \infty} ) |\mathcal{I}_L|   \log (C^* Kn)   }{ \delta^2}  \right)
	\]
operations to produce its output.
	\end{theorem}


	\begin{remark}
	Theorem \ref{thm:AnalysisALgoWithRound} precisely quantifies the benefit of truncation in Algorithm \ref{alg : multi-sinkhorn}: while all $K$ classes still play a role in the formation of an adversarial attack, by truncating the number of interactions between classes, in the worst case $|\mathcal{I}_L|$ scales like $\widetilde{O}(n^L)$ as opposed to the worst case of scaling of $|\mathcal{I}_K|$ which scales like $\widetilde{O}(n^K)$. 
	\end{remark}
	
	\section{Empirical results}\label{section : empirical results}
	
	\subsection{Numerical experiments for MNIST and CIFAR-10}

	In this section, we present experimental results obtained from applying our algorithms to datasets drawn from MNIST and CIFAR-10. For both data sets, there are 10 classes, and each class contains 100 points. Two different underlying ground metrics, $\ell^2$ and $\ell^{\infty}$, are used. Due to dimensional scaling effects, the adversary requires a larger budget when the ground metric is $\ell^2$. Note that  MNIST images are $28 \times 28$ pixel grayscale images, while CIFAR-10 images are $32 \times 32$ pixels with 3 color channels.

	Figure \ref{fig:MNIST, CIFAR10} shows the adversarial risks computed by the LP and Sinkhorn approaches using truncations of orders 2 and 3, along with their associated time complexities. Even though we restrict the interactions to order 2 or 3, the plots show that the adversarial risk does not change much when going from order 2 to order 3 interactions when the budget is not too large. This indicates that the truncated problem indeed provides meaningful lower bounds for the true problem when the adversarial budget is reasonable. Indeed, the curves for truncations of orders 2 and 3 are nearly identical for adversarial risk values below .3.  Let us emphasize that the adversarial risks obtained by the LP and Sinkhorn approaches do not coincide, and in fact, the former is always larger than the latter as Sinkhorn gives a lower bound for the true optimal adversarial risk. Here, we do not scale down the entropic parameter $\eta$ in terms of the number of points. One can reduce this gap by decreasing $\eta$, but this may cause numerical issues, which is a common phenomenon in computational optimal transport methods based on entropic regularization.

	We should emphasize that the size of the adversarial budget has an enormous impact on the computational complexities of both algorithms.  Indeed, both algorithms need the interaction matrix $\mathcal{I}_L$ constructed in Algorithm \ref{alg : contruct interactions} to be relatively sparse to run efficiently. For small values of $\varepsilon$ we expect $\mathcal{I}_L$ to be very sparse; however, as $\varepsilon$ increases, it will become more dense slowing down both algorithms.
 

    For the datasets that we consider, the worst-case complexities of the LP and Sinkhorn without truncation are $O(100^{30})$ and $\widetilde{O}(100^{10})$, respectively. With truncation up to order 3, the worst-case complexities become $O(100^9)$ and $\widetilde{O}(100^3)$, respectively. Note that these numbers are not usually achieved with adversarial budgets of small or moderate size (i.e. budgets that are relevant for adversarial training). Note, however, that the worst-case complexity of the LP approach can still be problematic even with truncation.  In practice, this does happen in our experiments once $\epsilon$ becomes sufficiently large, in this case, we terminate the computation once it exceeds a certain wall-clock time. For Sinkhorn, since its worst-case complexity is almost linear with respect to $n$, the computation remains feasible even for budgets where the LP approach is infeasible (at least when $\eta$ is not too small). This is one significant advantage of the Sinkhorn approach. However, one should still keep in mind that Sinkhorn will only return the exact value of the adversarial risk when the entropic regularization parameter is sent to zero (i.e. the regime where the algorithm gets slower and slower). Nonetheless, in our experiments,  with an appropriate choice of the entropic parameter $\eta$, the Sinkhorn solution is quite close to the exact solution provided by the LP, while offering a much shorter computation time.

   In Figure \ref{fig:comparison}, we run experiments on a smaller subset of the data where there are 4 classes with 50 points (for both MNIST and CIFAR-10).  This allows us to compare our computed value of the adversarial risk from the order 2 and 3 truncated problems to the computed value for the untruncated problem. Readers can observe from those plots that truncations of orders 2 and 3 barely underestimate or even match the full order 4 risk, especially when the adversarial budget isn't too large.  This is thanks to the fact that there are almost no valid interactions of higher order for reasonable values of $\varepsilon$.  Indeed in almost all of the plots, the order 2 truncated value matches the untruncated value when the adversarial risk is in the range of $0\%$ to $30\%$.  Once $\varepsilon$ is large enough that the risk grows beyond $.3$ we do start to see some discrepancy between the values of the truncated problem and the untruncated problem (especially when the ground metric is $\ell^{\infty}$). However, this regime is not so relevant as an adversarial risk of $20\%$ is already extremely large and suggests that one should be training with a smaller budget.

   From the experiments, we see that the truncation method works well in terms of both accurately approximating the adversarial risk and reducing the computational complexity. This is surprisingly nice because computing or even approximating the adversarial risk with many classes is hard in general: one must deal with a tensor of large order, requiring immense computational power. As long as classes are separated well, the truncation method will significantly reduce the order of tensors appearing in optimization (enhancing the efficiency of computing), while barely changing the optimal value of the problem, i.e. one should expect a very tight relaxation.

    
	
	{
	}


	\begin{figure}[h]
		\centering    \includegraphics[width=\linewidth]{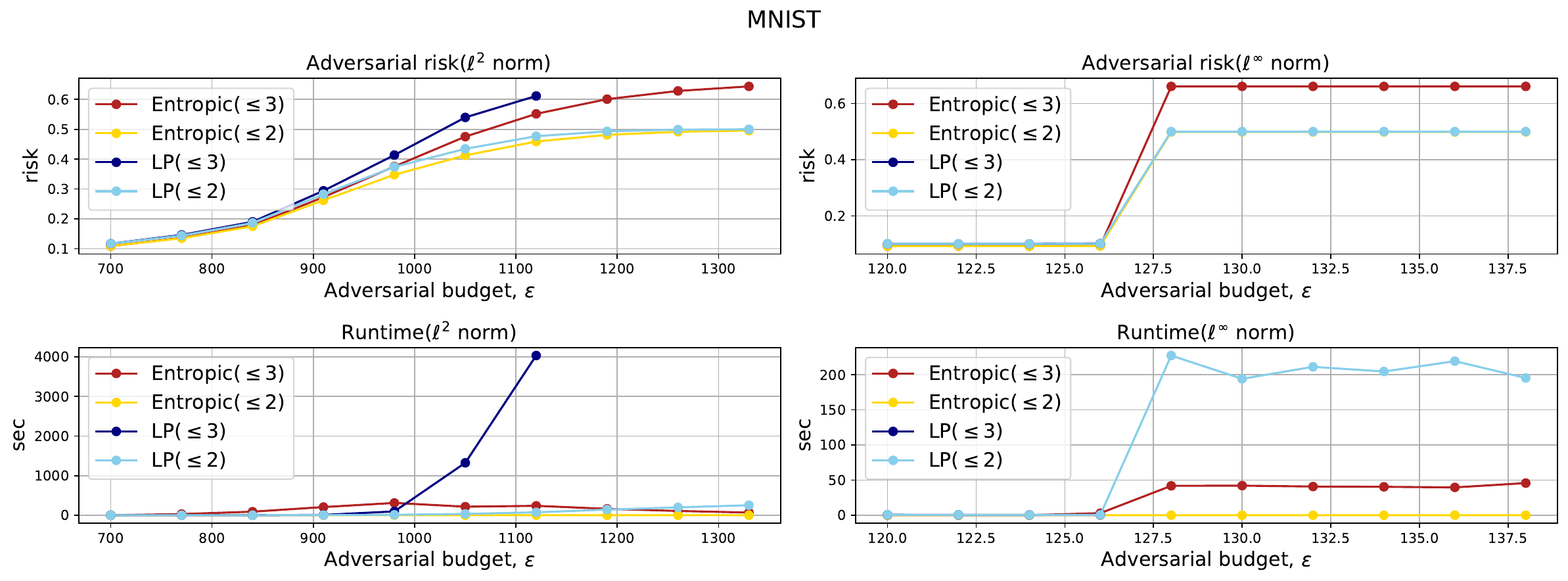}
		\centering    \includegraphics[width=\linewidth]{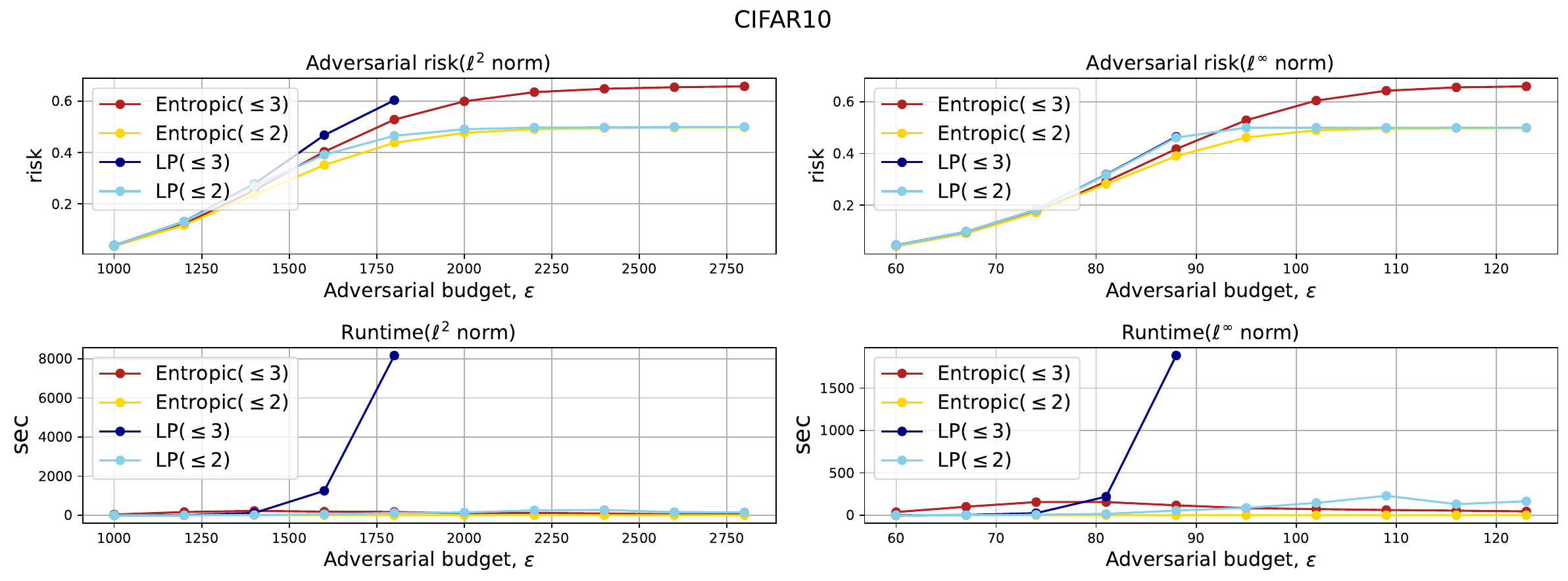}
		\caption{Lower bound of adversarial risk of and runtimes of the entropic regularization and LP for MNIST and CIFAR-10. The left plots and the right ones are equipped with $\ell^2$ norm and $\ell^{\infty}$ norm, respectively. For LP with truncation up to 3, due to the huge complexity we stop the computing earlier.}
		\label{fig:MNIST, CIFAR10}
	\end{figure}

	\begin{figure}[h]
		\centering
		\includegraphics[width=\linewidth]{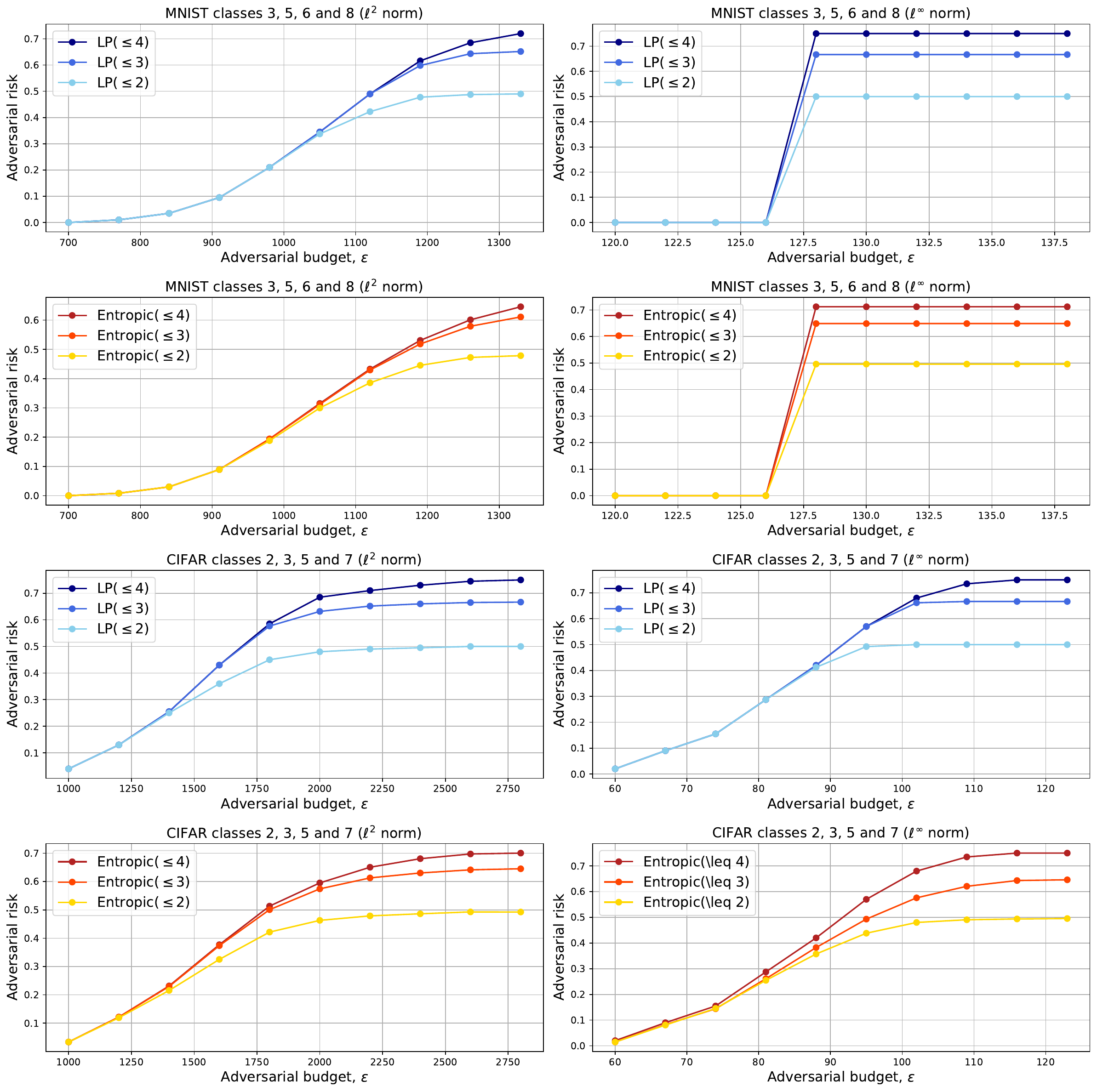}
		\caption{The optimal adversarial risks for MNIST and CIFAR-10 with 4 classes.}
		\label{fig:comparison}    
	\end{figure}

	\subsection{Fluctuation of interactions: Gaussian mixture, Iris and Glass data sets}

        In this section we further investigate in three settings the number of available higher-order interactions as well as how much mass the optimal multicoupling in \eqref{eq : stratified MOT} uses for each order. The three settings we consider are the following.
        
        \paragraph{Synthetic} We make a simple two-dimensional synthetic dataset which consists of six classes. For each class we generate 30 samples from $N(c_i, I)$, $2$-d Gaussian distribution with a mean $c_i$ and the identity covariance matrix, where $c_i \in \{(-2,2),(2,2),(6,2),(-2,-2),(2,-2),(6,-2)\}$. This gives a total of 180 samples. 

        \paragraph{Iris} This is the Iris dataset \cite{misc_iris_53} which is four dimensional (measurements of sepal length, sepal width, petal length, and petal width) and has three classifications (setosa, versicolor, and virginica). One must classify the type of iris from the four given measurements. There are 50 samples for each type of iris and a total of 150 samples.

        \paragraph{Glass} This is the Glass dataset \cite{misc_glass_identification_42} which is a ten dimensional (refractive index and percent composition of 9 atoms) and has six classifications (types of glass). There are 214 total samples non-uniformly distributed across the six classifications.
 
	In Figure \ref{fig : order contribution} we solve \eqref{eq : stratified MOT} and plot the amount of mass used in the optimal multicoupling, weighted by the number of interactions (we omit orders with negligible contributions). This places the curves for higher-order interactions on the same scale and represents how much total mass of the marginals is accounted for by each order of interaction.

    \begin{figure}[h]
        \centering
        \includegraphics[width=.5\linewidth]{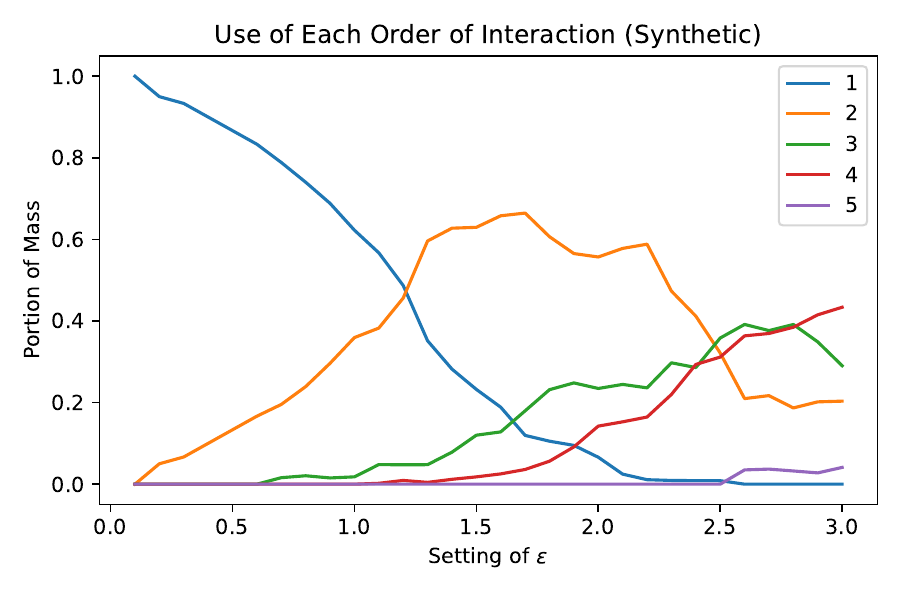}\hfill
        \includegraphics[width=.5\linewidth]{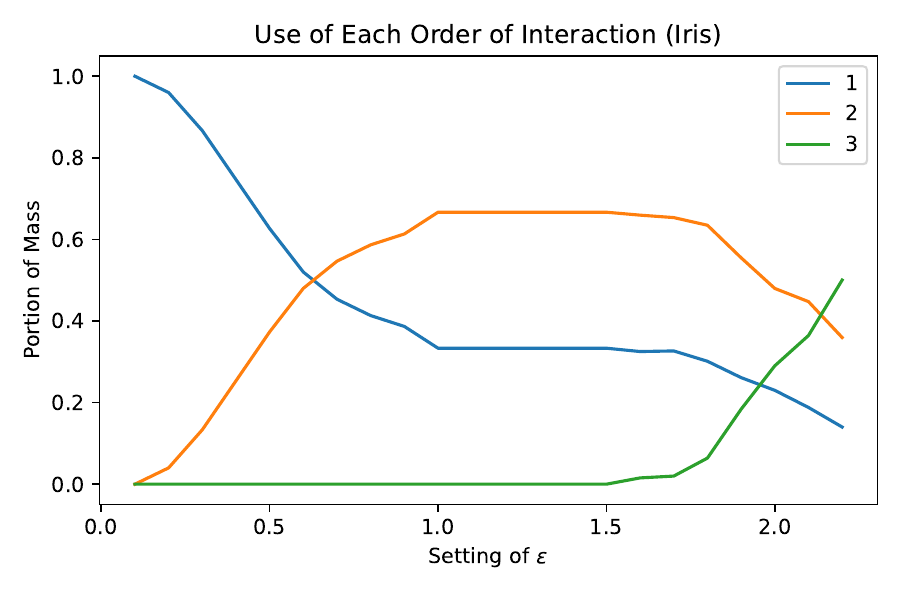}\hfill
        \includegraphics[width=.5\linewidth]{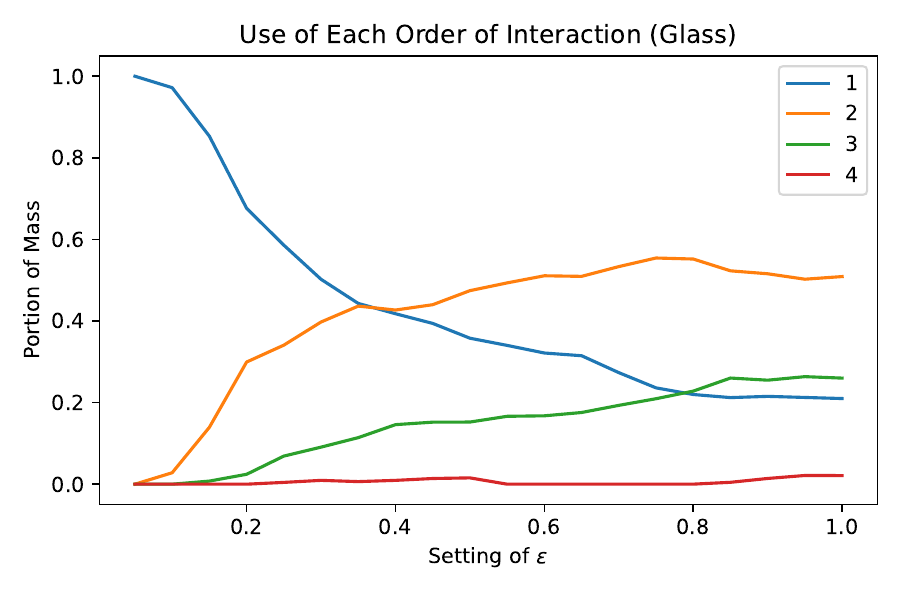}
        \caption{Contribution by interactions of each order to the optimal multicoupling as the budge $\epsilon$ varies in three different settings.}
        \label{fig : order contribution}
    \end{figure}

    Interestingly the use of lower-order interactions (high-order interactions respectively) does not monotonically decrease (increase). A simple example where this happens can be furnished using six points and three classes and is illustrated in Figure \ref{fig : less higher}. In this example if each point has unit weight the optimal perturbation for small budget has a mass of 4 (1 central point and 3 exterior points) and for a higher budge the optimal perturbation has a mass of 3 (3 midpoints) while using fewer order 3 interactions.

    \begin{figure}
        \centering
        \includegraphics[width=0.7\linewidth]{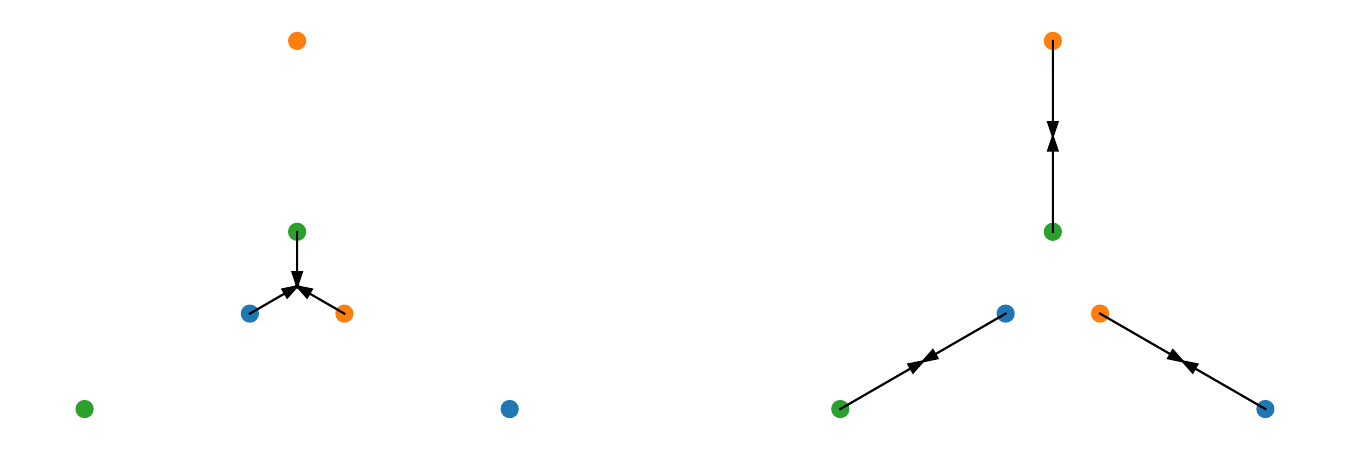}
        \caption{Configuration of six points with colors representing class. the central triangle has edge length 1 and the distance between corresponding points in the triangles is 2.0207. \textbf{(Left)} When $\varepsilon \in (0.5774, 1.0104)$ the optimal merging is achieved by placing the three interior points together. \textbf{(Right)} When $\varepsilon \in (1.0104, 1.0729)$ the optimal merging is achieved by pairing the interior triangle with the exterior triangle.} 
        \label{fig : less higher}
    \end{figure}

    In Figure \ref{fig : feasible inters} we also show the number of feasible interactions that must be considered of each order as the budget varies. For small budgets $\varepsilon$ there are typically only interactions of lower order. However there are sharp thresholds where the number of higher order interactions rapidly increases. After these thresholds the computational complexity of the optimization problems rapidly increases due to an explosion in the number of optimization variables.

    \begin{figure}[h]
        \centering
        \includegraphics[width=.5\linewidth]{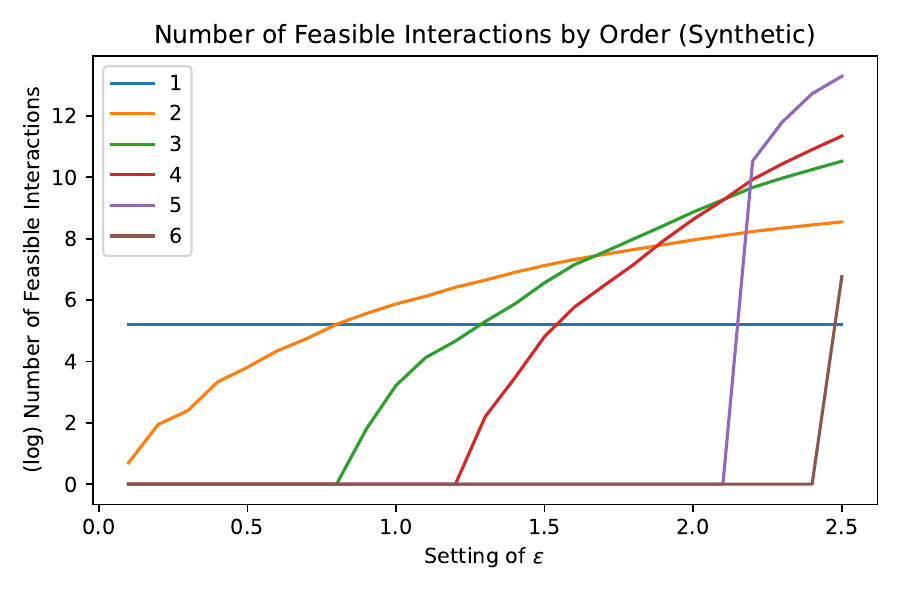}\hfill
        \includegraphics[width=.5\linewidth]{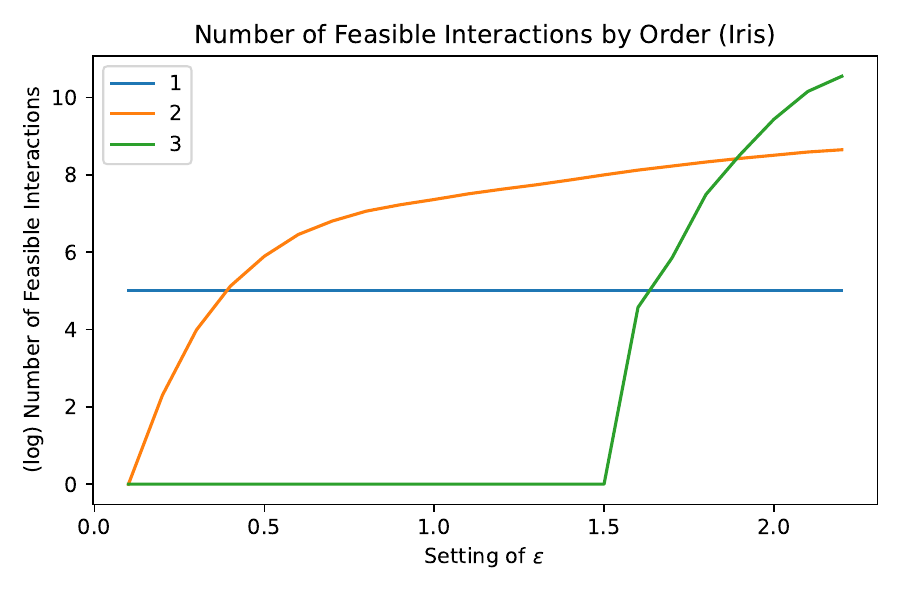}\hfill
        \includegraphics[width=.5\linewidth]{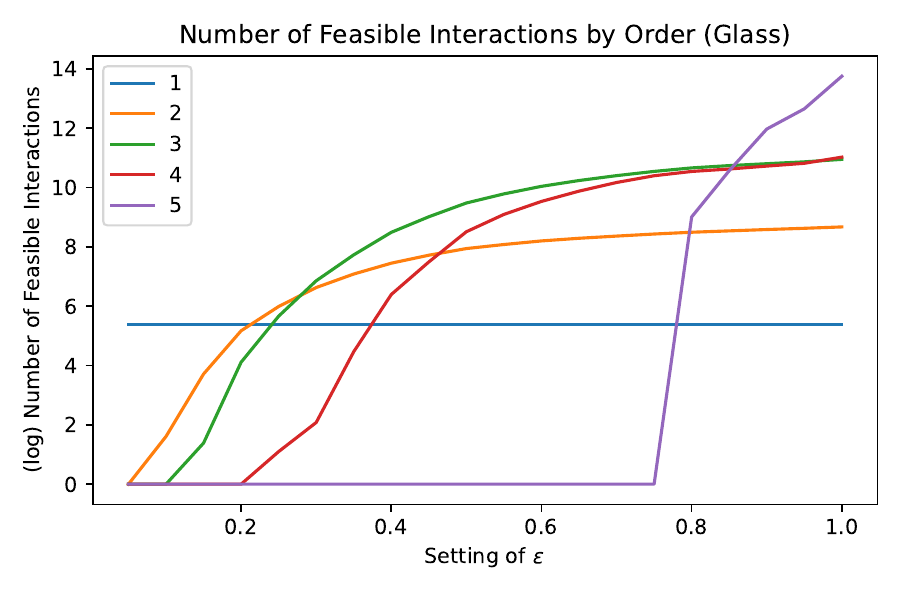}
        \caption{Number of interactions (plus one, in log scale) of each order which are feasible as the budget varies in the three different settings.}
        \label{fig : feasible inters}
    \end{figure}

    \subsection{Setting $\eta$ and $\delta'$ In Practice}\label{ssec:eta_delta}

    In Theorem \ref{thm:AnalysisALgoWithRound} we obtain bounds which achieve approximation ratios based on a parameter $\delta$. As one can see in Algorithm \ref{alg : entropic regularization} however, this requires setting $\eta$ and $\delta'$ as functions of $\delta$. This can be problematic in practice because taking $\eta$ too small (for example below 0.01) leads to numerical instability, and setting $\delta'$ too small may cause prohibitively many iterations in Algorithm \ref{alg : multi-sinkhorn}. Instead we often set $\eta$ to be sufficiently small, typically close to $\eta = 0.01$ and $\delta' = 0.001$. We observe empirically that these lead to high-quality solutions. This is illustrated on a synthetic example in Figure \ref{fig : eta impact}. 

    \begin{figure}[h]
        \centering
        \includegraphics[width=0.7\linewidth]{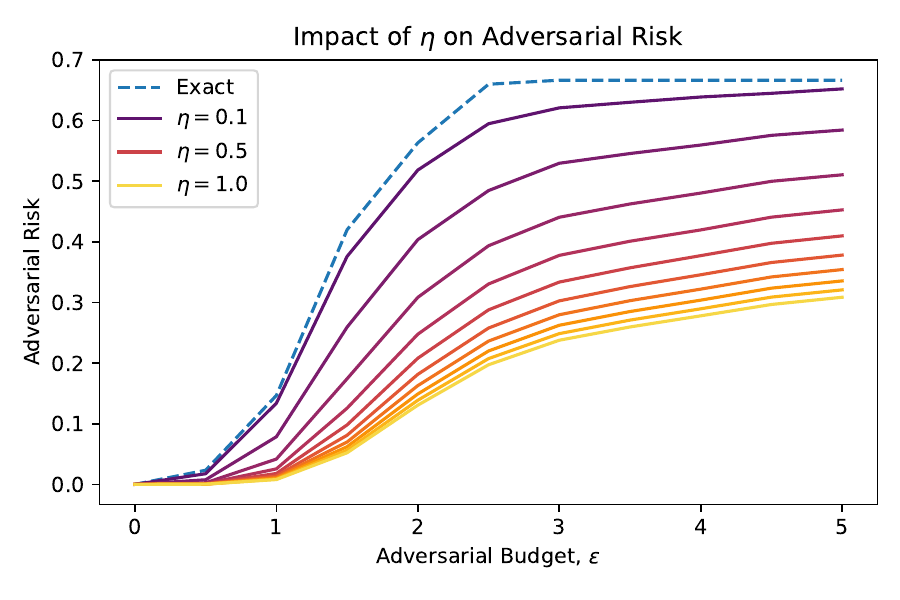}
        \caption{Impact on the adversarial risk as $\eta$ varies from $0.1$ to $1.0$ in steps of $0.1$ on synthetic data with $\delta' = 0.001$. }
        \label{fig : eta impact}
    \end{figure}

    To create this figure we used six classes in two dimensions with data drawn according to $N(\mu_i,I/2)$ where $\mu_i \in \{(-2,2),(2,2),(6,2),(-2,-2),(2,-2),(6,-2)\}$ and 100 samples from each class. The maximum allowed interaction size was three. The parameter $\delta'$ in Algorithm \ref{alg : multi-sinkhorn} was fixed at $0.001$. At least in this setting, a choice of moderate $\eta = 0.1$ achieves close to the exact adversarial risk (when only considering interactions up to size three). 
 	
	\section{Conclusion and future works}\label{section : conclusions and future works}
	In this paper, we propose two algorithms that demonstrate impressive performance in synthetic data sets by utilizing MOT formulations of the adversarial training problem. One notable aspect is our approach to reducing computational costs by truncating the order of interaction, due to the well separability of real data.

	A natural question is how to quantify the separation and impact of truncation on error. What is the ideal truncation level to reduce computational expenses while maintaining acceptable error levels? It would be highly beneficial to identify a sufficient criterion based on fundamental statistical measurements like mean, variance, and covariance to advise on the optimum truncation level. A Gaussian mixture model would be a promising starting point to address this inquiry.

	Improving the efficiency of both linear programming and the entropic regularization approaches is imperative. For the linear programming approach, it is possible to enhance it as the progress of understanding a general linear program. As for Algorithm \ref{alg : multi-sinkhorn}, it is known that the Sinkhorn algorithm can be parallelized in the binary setting(two marginals): see \cite{peyre2019computational}. However, to the best of our knowledge, it is unclear how to run Sinkhorn-type algorithms for MOT in parallel.  Developing an appropriately parallelized version of our algorithms would be a very interesting line of inquiry.

    \appendix
	\section{Appendix}\label{section : Appendix}
	
	\subsection{Proof of Proposition \ref{prop : approximation bounds}}
	\label{App:Prop1}
	
	\begin{proof}
		Let $(\lambda^*_A, \mu^*_{i,A})$ be the optimal measures in \eqref{eq : partition of barycenter}. Let $\lambda^L_A = \lambda^*_A$, $\mu^L_{i,A} = \mu^*_{i,A}$ for every $A$ with $|A| < L$ and let $\lambda^L_A = 0, \mu^L_{i,A} = 0$ for every $A$ with $|A| > L$.  The approach (made precise below) is for each $A$ with $|A| > L$ to distribute the mass in the sets $\lambda^*_A$ and $\mu^*_{i,A}$ uniformly over the subsets of size $L$ of $A$. 
		
		Now carrying out the details, for for every $A$ with $|A| = L$ let 
		\begin{align*}
		\lambda^L_A &= \sum_{\substack{B \in S_K \\ A \subset B}} \frac{|B|}{L} {|B| \choose L}^{-1} \lambda_B^* = \sum_{\substack{B \in S_K \\ A \subset B}}  {|B|-1 \choose L-1}^{-1} \lambda_B^* 
		\end{align*}
		and for every $i \in A$ set
		\begin{equation*}
		\mu_{i,A}^L = \sum_{\substack{B \in S_K \\ A \subset B}} {|B|-1 \choose L-1}^{-1} \mu_{i,B}^*.
		\end{equation*}
		Clearly the $(\lambda_A^L,\mu_{i,A}^L)$ do not place any mass on sets $A$ of size exceeding $L$. In addition 
		\begin{align*}
		\sum_{A \in S_K(i)} \mu_{i,A}^L &= \sum_{\substack{A \in S_K(i) \\ |A| < L}} \mu_{i,A}^* + \sum_{\substack{A \in S_K(i) \\ |A| = L}} \mu_{i,A}^L \\
		&= \sum_{\substack{A \in S_K(i) \\ |A| < L}} \mu_{i,A}^* + \sum_{\substack{A \in S_K(i) \\ |A| = L}}\sum_{\substack{B \in S_K \\ A \subset B}} {|B|-1 \choose L-1}^{-1} \mu_{i,B}^* \\
		&= \sum_{\substack{A \in S_K(i) \\ |A| < L}} \mu_{i,A}^* + \sum_{\substack{B \in S_K(i) \\ |B| \geq L}}  \sum_{ \substack{A \subset B \\ |A| = L \\ A \in S_K(i)}} {|B|-1 \choose L-1}^{-1} \mu_{i,B}^* \\
		&= \sum_{\substack{A \in S_K(i) \\ |A| < L}} \mu_{i,A}^* + \sum_{\substack{B \in S_K(i) \\ |B| \geq L}}  \mu_{i,B}^* = \mu_i.
		\end{align*}
		where we have used in order the definition of $\mu_{i,A}^L$, a change in the order of summation, that the inner summand is constant with $A$ and is counted precisely ${|B| - 1 \choose L - 1}$ times, and that $\mu_{i,A}^*$ is feasible so it must sum to $\mu_i.$ This shows that the $\mu_{i,A}^L$ terms also satisfy the summation constraint and are therefore feasible for \eqref{eq : truncation of barycenter}. 
		
		Next we will check that $C_\eps(\lambda_A, \mu_{i,A}^L) = 0$ for every $A$. Let $\pi_{i,A}^*$ be the couplings of $\lambda_A^*,\mu_{i,A}^*$ which achieve 
		\begin{equation*}
		0 = C_\eps(\lambda_A^*,\mu_{i,A}^*) = \int c_\eps(x,x') d\pi_{i,A}^*.
		\end{equation*}
		For every $A$ with $|A| < L$ we can use the coupling $\pi_{i,A}^*$ to couple $\mu_{i,A}^L$ and $\lambda_{A}^L$ since these equal $\mu_{i,A}^*$ and $\lambda_{A}^*$ respectively. For $|A| > L$ we can use $\pi_{i,A} = 0$ since $\mu_{i,A}^L = \lambda_{A}^L = 0$. All that remains is to handle $|A| = L$.
		
		In this case note that $\pi_{i,A}^*$ has first marginal $\mu_{i,A}^*$ and second marginal $\lambda_A^*$. From this it follows that if we define
		\begin{equation*}
		\pi_{i,A}^L = \sum_{\substack{B \in S_K \\ A \subset B}} {|B|-1 \choose L-1}^{-1} \pi_{i,B}^*
		\end{equation*}
		then $\pi_{i,A}^L$ will have marginals $\mu_{i,A}^L$ and $\lambda_{A}^L$. We can also check that the cost is zero via
		\begin{equation*}
		\int c_\eps(x,x') d\pi_{i,A}^L(x,x') = \sum_{\substack{B \in S_K \\ A \subset B}} {|B|-1 \choose L-1}^{-1} \int c_\eps(x,x') d\pi_{i,B}^*(x,x') =\sum_{\substack{B \in S_K \\ A \subset B}} {|B|-1 \choose L-1}^{-1} 0.
		\end{equation*}
		This shows that $\pi_{i,A}^L$ is a coupling of $\mu_{i,A}$ and $\lambda_A$ with zero cost. Finally we can compare the objective costs of $(\lambda_A^*,\pi_{i,A}^*)$ with $(\lambda_A^L, \pi_{i,A}^L)$ as follows
		\begin{align*}
		&\sum_{A \in S_K}  \left( \lambda_A^L(\mathcal{X}) + \sum_{i \in A} C_\eps(\lambda_A^L, \mu_{i, A}^L) \right) - \left( \lambda_A^*(\mathcal{X}) + \sum_{i \in A} C_\eps(\lambda_A^*, \mu_{i, A}^*) \right)\\
		&= \sum_{\substack{A \in S_K \\ |A| = L}} \lambda_A^L(\X) - \sum_{\substack{A \in S_K \\ |A| \geq L}} \lambda_A^*(\X) \\
		&= \sum_{\substack{A \in S_K \\ |A| = L}} \left ( \sum_{\substack{B \in S_K \\ A \subset B}} \frac{|B|}{L} {|B| \choose L}^{-1} \lambda_B^*(\X) \right )  - \sum_{\substack{A \in S_K \\  |A| \geq L}} \lambda_A^*(\X) \\
		&= \sum_{\substack{ A \in S_K \\ |A| \geq L}} \left ( \frac{|A|}{L} - 1 \right ) \lambda_A^*(\X) = \sum_{k > L} \left ( \frac{k}{L} - 1 \right ) \sum_{|A| = k} \|\lambda_A^*\|.
		\end{align*}
		In the jump to the second line we have used that the $C_\eps$ terms are all zero, that $\lambda_A^L = \lambda_A^*$ for $|A| < L$ and that $\lambda_A^L = 0$ for $|A| > L$. The third line uses the definition of $\lambda_A^L$. The final line is a term counting argument. This completes the first part of the proof. 
		
		The second part of the proof uses an analogous treatment which we only sketch. Let $(\pi_A^*)$ be the optimal multicouplings in \eqref{eq : stratified MOT} and define $(\pi_A^L)$ by taking $\pi_A^L = \pi_A^*$ for $|A| < L$, $\pi_A^L = 0$ for $|A| > L$ and 
		\begin{equation*}
		\pi_A^L = \sum_{\substack{B \in S_K \\ A \subset B}} \frac{|B|}{L} {|B| \choose L}^{-1} \lambda_B^* = \sum_{\substack{B \in S_K \\ A \subset B}}  {|B|-1 \choose L-1}^{-1} \mathcal{P}_A \pi_B^* 
		\end{equation*}
		where $\mathcal{P}_A \pi_B^*$ is the projection of $\pi_B^*$ onto its marginals corresponding to the set $A$. The remainder of the proof follows essentially the same structure once we observes $c_{\eps,A} \leq c_{\eps,B}$ for all $A \subset B$ which is helpful for showing that $\pi_A^L$ has finite cost.
	\end{proof}

	\subsection{Dual of \eqref{eq : truncated entropic problem}}
	\label{app:Dual}
	
	The Lagrangian for problem \eqref{eq : truncated entropic problem} is  
	\begin{equation*}
	\begin{aligned}
	& \mathcal{L}( \{\pi_A\}_{A \in S_K^L}, \{g_i\}_{i \in \Y} )\\
	&:= \sum_{A \in S_K^L} \sum_{\mathcal{X}^A} \left ( 1 + c_A(x_A)  \right ) \pi_A(x_A) - \eta H(\pi_A) - \sum_{i \in \mathcal{Y}} \sum_{\mathcal{X}_i} g_i (x_i) \left( \sum_{A \in S_K^L(i)} {\mathcal{P}_i}_{\#} \pi_A(x_i) - \mu_i(x_i) \right),
	\end{aligned}    
	\end{equation*}
	where $\{ g_i \in \mathbb{R}^{n_i}\}_{i \in \Y}$ is the collection of dual variables (one for each marginal constraint). The corresponding dual objective is defined as:
	\begin{equation}\label{eq:Lagrangian}
	\mathcal{G}(\{ g_i \} _{i \in \Y}) :=  \min_{\{\pi_A\}_{A \in S_K^L}} \mathcal{L}( \{\pi_A\}, \{g_i\} ).  
	\end{equation}
	Notice that for every fixed $\{ g_i\}$ \eqref{eq:Lagrangian} is a strictly convex optimization problem and thus its first order optimality conditions are sufficient to guarantee optimality. In turn, a straightforward computation shows that these first order optimality conditions read:
	\begin{equation*}\label{eq:first order condition}
	0=\partial_{\pi_A(x_A)} \mathcal{L}( \{\pi_A\}, \{g_i\} ) = 1 + c_A(x_A) + \eta \log \pi_A(x_A) - \sum_{i \in A} g_i(x_i),  \quad \forall x_A \in \X^A , \quad \forall A \in S_K^L.
	\end{equation*}
	As a result, we conclude that the unique solution of \eqref{eq:Lagrangian} is given by $\{ \pi_A(g) \}_{A\in S_K^L}$, the set of couplings of the form \eqref{eq:PiFromg}.

	Since
	\begin{align*}
	\sum_{i \in \mathcal{Y}} \sum_{\mathcal{X}_i} g_i (x_i) \sum_{A \in S_K^L(i)} {\mathcal{P}_i}_{\#} \pi_A(x_i) &= \sum_{i \in \mathcal{Y}} \sum_{A \in S_K^L(i)} \sum_{\mathcal{X}^A} g_i(x_i) \pi_A(x_A)\\
	&= \sum_{A \in S_K^L} \sum_{ \mathcal{X}^A} \left( \sum_{i \in A}g_i(x_i) \right) \pi_A(x_A),
	\end{align*}
	it follows that the dual of \eqref{eq : truncated entropic problem} is the maximization problem
	\begin{equation*}
	\begin{aligned}
	&\max_{ \{g_i\}_{i \in \Y}} 
	& \sum_{i \in \mathcal{Y}} \sum_{\mathcal{X}_i} g_i (x_i) \mu_i(x_i) - \eta  \sum_{A \in S_K^L} \sum_{\mathcal{X}^A} \exp \left( \frac{1}{\eta} \left( \sum_{i \in A} g_i(x_i) - \left(1 + c_A(x_A) \right) \right) \right).
	\end{aligned}
	\end{equation*}
	The above is equivalent to the minimization problem \eqref{eq:truncated_obj}.

	\subsection{Analysis of Algorithm \ref{alg : multi-sinkhorn}}
	\label{app:AnalysisAlg4}
	
	Our goal in this section is to prove Theorem \ref{thm : iteration}. In preparation for its proof, we state and prove a series of auxiliary results. We start recalling the definition of Kulleback-Leibler divergence between measures with possibly different total masses. 
	
	\begin{definition}
		Given two finite positive measures $\mu$ and $\nu$ (not necessarily with the same total mass) sharing a common finite support $\mathcal{Z}$, their KL divergence is defined as
		\begin{equation*}
		D_{\text{KL}}(\mu || \nu) := \sum_{z \in \mathcal{Z}} \left( \nu(z) - \mu(z) \right) + \sum_{z \in \mathcal{Z}} \mu(z) \log \frac{\mu(z)}{\nu(z)}.
		\end{equation*}
		Notice that when $\mu$ and $\nu$ are probability measures, the above definition coincides with the usual one. Also, like for the usual KL-divergence, $D_{\text{KL}}(\mu || \nu)$ is non-negative and is equal to $0$ if and only if $\mu=\nu$.
	\end{definition}
	
	The next lemma is a variant of \cite[Lemma 6]{altschuler2017near} adapted to the setting where $\mu$ and $\nu$ are allowed to have different total masses.
	
	\begin{lemma}
		Let $\mu$ and $\nu$ be finite positive measures over a finite set $\mathcal{Z}$ such that $\mu \leq \nu$. If $D_{\text{KL}}(\mu || \nu) \leq ||\mu||_1$, then
		\begin{equation}\label{eq : Pinkser}
		D_{\text{KL}}(\mu || \nu) \geq \frac{1}{7 ||\mu||_1} ||\mu - \nu ||_1^2.
		\end{equation}    
	\end{lemma}

	\begin{proof}
		Let $\overline{\mu}, \overline{\nu}$ be normalized probability vectors obtained from $\mu$ and $\nu$, respectively. One can write
		\begin{align*}
		D_{\text{KL}}(\mu || \nu) &= ||\nu||_1 - ||\mu||_1 + \sum_{z \in \mathcal{Z}} \mu(z) \log \frac{\mu(z)}{\nu(z)}\\
		&=||\mu||_1 \log ||\mu||_1 + ||\nu||_1 - ||\mu||_1 - ||\mu||_1 \log ||\nu||_1 + ||\mu||_1 D_{\text{KL}}(\overline{\mu} || \overline{\nu})\\
		&= ||\mu||_1 \left( \frac{||\nu||_1}{||\mu||_1} - 1 - \log \frac{||\nu||_1}{||\mu||_1} +  D_{\text{KL}}(\overline{\mu} || \overline{\nu}) \right).
		\end{align*}
		Note that $\frac{||\nu||_1}{||\mu||_1} - 1 - \log \frac{||\nu||_1}{||\mu||_1}$ and $D_{\text{KL}}(\overline{\mu} || \overline{\nu})$ are both non-negative. In particular, if $D_{\text{KL}}(\mu || \nu) \leq ||\mu||_1$, then    
		\[
		\frac{||\nu||_1}{||\mu||_1} - 1 - \log \frac{||\nu||_1}{||\mu||_1} \leq 1.
		\]
		With the aid of some basic calculus and algebra one can show that for those $s \in (0,\infty)$ satisfying $s-1 -\log(s) \leq 1$ one has the lower bound $s - 1 - \log(s) \geq (s-1)^2 /5$. Therefore,
		\[
		\frac{||\nu||_1}{||\mu||_1} - 1 - \log \frac{||\nu||_1}{||\mu||_1} \geq \frac{(\frac{||\nu||_1}{||\mu||_1} - 1 )^2 }{5}.
		\]

		An application of Pinsker's inequality for probability measures yields
		\[
		D_{\text{KL}}(\mu || \nu) \geq ||\mu||_1 \left( \frac{(\frac{||\nu||_1}{||\mu||_1} - 1 )^2 }{5} + \frac{|| \overline{\mu} - \overline{\nu} ||_1^2}{2} \right).
		\]
		Finally, by the triangle inequality and Young's inequality,
		\begin{align*}
		||\mu - \nu||_1^2 &= ||\mu||^2_1 || \overline{\mu} - \frac{\nu}{||\mu||_1} ||_1^2\\
		&\leq ||\mu||_1^2 \left( ||\overline{\nu} - \frac{\nu}{||\mu||_1} ||_1 + || \overline{\mu} - \overline{\nu} ||_1 \right)^2\\
		&= ||\mu||_1^2 \left( \left|\frac{||\nu ||_1}{||\mu||_1} - 1 \right| + || \overline{\mu} - \overline{\nu} ||_1 \right)^2\\
		&\leq \frac{7 ||\mu||_1^2}{5}  \left( \frac{||\nu ||_1}{||\mu||_1} - 1 \right)^2 + \frac{7 ||\mu||_1^2}{2} || \overline{\mu} - \overline{\nu} ||_1^2.
		\end{align*}
		This completes the proof.    
	\end{proof}

	In what follows we use $\langle \cdot, \cdot \rangle$ to denote the inner product of any two vectors in the same Euclidean space. In particular, if $\Z$ is a finite set, $h: \Z \rightarrow \R$, and $\nu$ is a measure over $\Z$, then $\langle h, \nu \rangle := \sum_{z \in \mathcal{Z}} h(z) \nu(z)$. We provide a lower bound on the decrement of the energy $\G^L$ along the iterates in Algorithm \ref{alg : multi-sinkhorn}.
	
	\begin{proposition}
		\label{prop:LowerBoundGap}
		Let $\{g^t\}_{t \in \mathbb{N}}$ be generated by Algorithm \ref{alg : multi-sinkhorn} and let $\T$ be the collection of iterates $t$ for which the following holds:
		\begin{equation}\label{eq : bound of KL-divergence}
		D_{\text{KL}} (\mu_i || \sum_{A \in S^L_K(i)} {\mathcal{P}_i}_{\#} \pi_A(g^t)) \leq \lVert\mu_i\rVert_1 \quad \forall i \in \mathcal{Y}.
		\end{equation}
		Then the following holds for any iterate $t$:
		\begin{equation}
		\G^L(g^t) - \G^L(g^{t+1}) \geq \frac{1}{7} \left( \frac{E_t}{K} \right)^2, \quad \text{if } t \in \mathcal{T}, 
		\end{equation}
		and
		\begin{equation}
		\G^L(g^t) - \G^L(g^{t+1}) \geq \min_{j \in \Y} \lVert \mu_j \rVert_1, \quad \text{if } t \not \in \T.
		\label{eq:EnergyDecreaseNoT}
		\end{equation}
	\end{proposition}
	\begin{proof}
		Consider an iterate $t$ and let $I$ be the greedy coordinate at $t+1$ in \textbf{Step 1} of Algorithm \ref{alg : multi-sinkhorn}.  \textbf{Step 2} of Algorithm \ref{alg : multi-sinkhorn} produces
		\begin{equation}\label{eq : update g^t}
		g^{t+1}_I = g^{t}_I + \eta \log \mu_I - \eta \log \sum_{A \in S^L_K(I)} {\mathcal{P}_I}_{\#} \pi_A(g^{t}).
		\end{equation}
		It is straightforward to see that
		\begin{align*}
		&\G^L(g^t) - \G^L(g^{t+1})\\
		&=\sum_{x_I \in \mathcal{X}_I} \left(\sum_{A \in S^L_K(I)} {\mathcal{P}_I}_{\#} \pi_A(g^t)(x_I) - \sum_{A \in S^L_K(I)} {\mathcal{P}_I}_{\#} \pi_A(g^{t+1})(x_I) \right) - \langle \frac{g_I^t}{\eta} - \frac{g^{t+1}_I}{\eta}, \mu_I \rangle\\
		&=\sum_{x_I \in \mathcal{X}_I} \left( \sum_{A \in S^L_K(I)} {\mathcal{P}_I}_{\#} \pi_A(g^t)(x_I) - \mu_I(x_I) \right) + \langle \log \mu_I - \log \sum_{A \in S^L_K(I)} {\mathcal{P}_I}_{\#} \pi_A(g^t), \mu_I \rangle\\
		&= D_{\text{KL}} (\mu_I || \sum_{A \in S^L_K(I)} {\mathcal{P}_I}_{\#} \pi_A(g^t))\\
		&\geq D_{\text{KL}} (\mu_i || \sum_{A \in S^L_K(i)} {\mathcal{P}_i}_{\#} \pi_A(g^t)), \quad \forall i \in \Y.
		\end{align*}
		At this stage we split the analysis into two cases. First, if we assume that $t \not \in \T$, then there is an $i \in \Y$ for which $D_{\text{KL}} (\mu_i || \sum_{A \in S^L_K(i)} {\mathcal{P}_i}_{\#} \pi_A(g^t)) > \lVert \mu_i \rVert_1$. In particular, this implies \eqref{eq:EnergyDecreaseNoT}. 	On the other hand, if $t \in \T $, we can use the above chain of inequalities to obtain
		\begin{align*}
		\G^L(g^t) - \G^L(g^{t+1})\geq \frac{1}{K} \sum_{i \in \mathcal{Y}} D_{\text{KL}} (\mu_i || \sum_{A \in S^L_K(i)} {\mathcal{P}_i}_{\#} \pi_A(g^t))
		\end{align*}
		and apply \eqref{eq : Pinkser}, \eqref{eq : bound of KL-divergence}, and the fact that $\lVert \mu_i \rVert_1 \leq 1$ for all $i \in \mathcal{Y}$ to deduce
		\[
		\G^L(g^t) - \G^L(g^{t+1}) \geq  \frac{1}{7} \sum_{i \in \mathcal{Y}}  \frac{1}{K} \lVert  \mu_i - \sum_{A \in S^L_K(i)} {\mathcal{P}_i}_{\#} \pi_A(g^t)) \rVert_1^2.
		\]
		Applying Jensen's inequality, we deduce
		\[
		\G^L(g^t) - \G^L(g^{t+1}) \geq  \frac{1}{7} \left( \sum_{i \in \mathcal{Y}}  \frac{1}{K}\lVert \mu_i - \sum_{A \in S^L_K(i)} {\mathcal{P}_i}_{\#} \pi_A(g^t)) \rVert_1 \right)^2= \frac{1}{7}\left(\frac{E_t}{K} \right)^2 .
		\]

	\end{proof}




	Next, we find an upper bound for the energy gap between $g^t$ as generated by Algorithm \ref{alg : multi-sinkhorn} and an optimal $g^*$. 
	\begin{proposition}
		\label{prop : upper bound of multi sinkhorn}
		Let $\{g^t\}_{t \in \N}$ be generated by Algorithm \ref{alg : multi-sinkhorn} and let $g^*$ be a minimizer for \eqref{eq:truncated_obj}. Then, for all $t \geq 1$, 
		\[ \G^L(g^t) - \G^L(g^*) \leq \frac{\overline{R}}{\eta} E_t,\]
		where we recall $E_t$ was defined in \eqref{eq : stopping criterion} and where $\overline{R}$ was defined in \eqref{def:OvR}. 
	\end{proposition}
	
	In order to prove Proposition \ref{prop : upper bound of multi sinkhorn} we first prove some auxiliary estimates.

	\begin{lemma}
		\label{lem:AuxAtOptimal}
		Let $g^*$ be a minimizer for \eqref{eq:truncated_obj}. Then, for all $i \in \mathcal{Y}$,

		\begin{equation}\label{eq : bound of g^*}
		 -(L-1) +  \eta  \log \min_{j \in \mathcal{Y}, y \in \mathcal{X}_j} \mu_j(y) - \eta L\log(K  C^* n)     \leq \min_{x_i \in \X_i } g_i^*(x_i )  \leq \max_{x_i \in \X_i } g_i^*(x_i ) \leq 1.	
		\end{equation}
	\end{lemma}

	\begin{proof}
		The first order optimality conditions for $g^*$ imply that, for each $i \in \mathcal{Y}$ and $x_i\in \X_i$, 
		
		\begin{equation}
		\sum_{A \in S^L_K(i)} {\mathcal{P}_i}_{\#} \pi_A(g^*)(x_i)= \mu_i(x_i) \geq \min_{j \in \mathcal{Y}, y \in \mathcal{X}_j} \mu_j(y).
		\label{eqn:AuxLemma6}
		\end{equation}
		Expanding the left-hand side of the above inequality we get
		\begin{equation}\label{eq : expanding exp g^* = mu}
		\begin{aligned}
		&\sum_{A \in S^L_K(i)} {\mathcal{P}_i}_{\#} \pi_A(g^*)(x_i)\\
		&=  \exp \left(\frac{1}{\eta} \left( g^*_i(x_i) - (1 + c_{\{i\}}(x_i) ) \right) \right)\\
		&\quad + \exp \left(\frac{1}{\eta} g^*_i(x_i) \right) \sum_{A \neq \{i \} \in S^L_K(i)} \sum_{x_{A \setminus \{i \}} \in \mathcal{X}^{A \setminus \{i\} }} \exp \left( \frac{1}{\eta} \sum_{j \in A} g^*_j(x_j) - (1 +c_A(x_A) ) \right)\\
		& \geq \exp \left(\frac{1}{\eta} \left( g^*_i(x_i) - 1 \right) \right)
		\end{aligned}
		\end{equation}
		since the double summation in the second line is always non-negative and $c_{\{i\}}(x_i) = 0$. Hence,
		\[
		\exp \left(\frac{1}{\eta} \left( g^*_i(x_i) - 1 \right) \right) \leq  \sum_{A \in S^L_K(i)} {\mathcal{P}_i}_{\#} \pi_A(g^*)(x_i) = \mu_i(x_i).
		\]
		Taking logarithms, it follows 
		\[
		\frac{1}{\eta} g^*_i(x_i) - \frac{1}{\eta} \leq \log \mu_i(x_i) \leq 0.
		\]
		Thus,
		\[
		\max_{x_i \in \X_i} g^*_i(x_i) \leq 1.
		\]
		To get a lower bound for $g^*_i$, 	we can take logarithms in \eqref{eqn:AuxLemma6} and use the fact that $c_A \geq 0$ to deduce
		\[
		\frac{g_i^*(x_i)}{\eta} \geq \log \min_{j \in \mathcal{Y}, y \in \mathcal{X}_j} \mu_j(y) - \log \sum_{A \in S^L_K(i)} \sum_{x_{A \setminus \{i \}} } \exp \left( \frac{1}{\eta} \sum_{ j \in A \setminus \{i\}} g^*_j(x_j) \right) + \frac{1}{\eta}.
		\]
		In turn, since we already know that $\max_{x_j} g_j^*(x_j) \leq 1 $ for all $j \in \Y$, we can further obtain
		\[   \frac{g_i^*(x_i)}{\eta} \geq \log \min_{j \in \mathcal{Y}, y \in \mathcal{X}_j} \mu_j(y) - \log(K^L  (C^* n)^L \exp(L/\eta ) )  + \frac{1}{\eta}  = \log \min_{j \in \mathcal{Y}, y \in \mathcal{X}_j} \mu_j(y) - L\log(K  C^* n)   - \frac{(L-1)}{\eta}. \]
		The desired lower bound now follows.

%
	\end{proof}

	\begin{lemma}\label{lem : bound of iterated g}
		Let $\{g^t\}_{t \in \mathbb{N}}$ be generated by Algorithm \ref{alg : multi-sinkhorn}. Then, for all $t \geq 1$ and all $i \in \mathcal{Y}$,
		\begin{equation}\label{eq : bound of g^t}
		 -(L-1) +  \eta  \log \min_{j \in \mathcal{Y}, y \in \mathcal{X}_j} \mu_j(y) - \eta L\log(K  C^* n)     \leq \min_{x_i \in \X_i } g_i^t(x_i )  \leq \max_{x_i \in \X_i } g_i^t(x_i ) \leq 1.	
		\end{equation}
	\end{lemma}

	\begin{proof}
		Since $g^0_i = \boldsymbol{0}$ for all $i \in \mathcal{Y}$, \eqref{eq : bound of g^t} holds trivially in the case $t=0$. Assume that \eqref{eq : bound of g^t} holds for all $s \leq t - 1$. Let $I$ be the greedy coordinate chosen in \textbf{Step 1} of Algorithm \ref{alg : multi-sinkhorn} at $t-1$. For all $i \neq I$, the induction hypothesis implies \eqref{eq : bound of g^t}. On the other hand, thanks to \eqref{eq : update g^t}, it follows
		\[
		\sum_{A \in S^L_K(I)} {\mathcal{P}_I}_{\#} \pi_A(g^t)(x_I) = \mu_I(x_I) \geq \min_{j \in \mathcal{Y}, y \in \mathcal{X}_j} \mu_j(y) .
		\]
		Taking logarithms and using the fact that $c_A \geq 0$ we obtain 
		\[
		\frac{g_I^t(x_I)}{\eta} \geq \log \min_{j \in \mathcal{Y}, y \in \mathcal{X}_j} \mu_j(y) - \log \sum_{A \in S^L_K(I)} \sum_{ x_{A \setminus \{I \}} } \exp \left( \frac{1}{\eta} \sum_{j \in A \setminus \{I\}} g^{t-1}_j(x_j) \right) + \frac{1}{\eta}.
		\]
		We can then use the induction hypothesis $\max_{x_j } g_j^{t-1}(x_j) \leq 1$ for all $j$ to get the desired lower bound. 
		
		For the upper bound, we notice that
		\begin{equation*}
		\begin{aligned}
		\mu_I(x_I) & \geq \sum_{A \in S^L_K(I)} {\mathcal{P}_I}_{\#} \pi_A(g^t)(x_I)\\
		&=  \exp \left(\frac{1}{\eta} \left( g^t_I(x_I) - (1 + c_{\{I\}}(x_I) ) \right) \right)\\
		&\quad +  \sum_{A \neq \{i \}} \sum_{x_{A }} \exp \left( \frac{1}{\eta} \sum_{j \in A} g^t_j(x_j) - (1 +c_A(x_A) ) \right)\\
		& \geq \exp \left(\frac{1}{\eta} \left( g^t_I(x_I) - 1 \right) \right),
		\end{aligned}
		\end{equation*}
		from where we can deduce that $g_I(x_I) \leq 1$ for all $x_I$.	\end{proof}

	We are ready to prove Proposition \ref{prop : upper bound of multi sinkhorn}.

	\begin{proof}[Proof of Proposition \ref{prop : upper bound of multi sinkhorn}]
		Recall that $\G^L$ is convex and differentiable. We thus have
		\[
		\G^L(g^t) - \G^L(g^*) \leq \langle g^t - g^*, \nabla_g  \G^L(g^t)  \rangle = \sum_{i=1}^K \langle g_i^t - g_i^*, \partial_{g_i} \G^L(g^t)  \rangle.
		\]
		For notational convenience, in the remainder of this proof we use
		\[
		P_i^t := \sum_{A \in S^L_K(i)} {\mathcal{P}_i}_{\#} \pi_A(g^t).
		\]
		Since $\partial_{g_i} \G^L(g^t) = \frac{1}{\eta} \left( P_i^t - \mu_i \right)$ (as can be seen from a direct computation), we obtain
	\begin{align*}
	\G^L(g^t) - \G^L(g^*) &\leq \sum_{i=1}^K \frac{1}{\eta} \langle g_i^t - g_i^*, P_i^t - \mu_i  \rangle\\
	& \leq \frac{\overline{R}}{\eta}E_t , 
	\end{align*}
		where we have used \eqref{eq : bound of g^*} and \eqref{eq : bound of g^t} to bound $ \lVert g_i^* - g_i^t \rVert_\infty $ by $\overline{R}$.
	
	\end{proof}

	%

	%

	\begin{lemma}
		\label{lem:NiceSequence}
		Suppose that $\{ a_t \}_t$ is a decreasing sequence of positive numbers (finite or infinite) satisfying:
		\[ a_{t} - a_{t+1} \geq \max \{  B\delta'^2 , A a_t^2  \}, \quad \forall t,   \]
		for positive constants $A,B, \delta'$. Then for all $T$ we have
		\begin{equation}
		\label{eq:IneqAux}
		T \leq \min_{ h \text{ s.t. } h \geq a_T } \left( 2+ \frac{1}{Ah} +  \frac{h}{B\delta'^2} \right). 
		\end{equation}
		
	\end{lemma}
	\begin{proof}
		Let $t' \leq  T$. From the fact that $a_{t+1} -a_t \geq B \delta'^2$ for all $t$ we have
		\[ a_{t'} \geq a_{t'} - a_{T} =  \sum_{t=t'}^{T-1}  (a_{t} - a_{t+1})\geq \sum_{t=t'}^{T-1} B \delta'^2 \geq B \delta'^2 (T- t').        \]
		On the other hand, from $a_t - a_{t+1} \geq A a_t^2$ we get 
		\[ \frac{1}{a_{t'}} \geq \frac{1}{a_t'} - \frac{1}{a_{1}} = \sum_{t=1}^{t'-1}  \left( \frac{1}{a_{t+1}} - \frac{1}{a_t} \right) =  \sum_{t=1}^{t'-1} \frac{a_t - a_{t+1} }{ a_t a_{t+1} } \geq A \sum_{t=1}^{t'-1} \frac{a_t}{a_{t+1}} \geq A(t'-1),  \]
		using the fact that, by the assumptions, $a_t$ is a decreasing sequence. 
		
		Combining the above inequalities we get
		\[  T-1 = (T-t') + (t'-1) \leq \frac{1}{A a_{t'}} + \frac{a_{t'}}{ B \delta'^2 }. \]
		In particular,
		\[ T-1 \leq \min_{t'=1, \dots, T} \left(  \frac{1}{A a_{t'}} + \frac{a_{t'}}{B \delta'^2} \right). \]

		To be able to obtain \eqref{eq:IneqAux} we need to modify the above argument slightly. We will show that for any $h \geq a_T $ we have
		\[  T -2 \leq \frac{1}{A h } + \frac{h}{B \delta'^2}. \]
		
		First, consider $h \in [a_{t'+1}, a_{t'}]$ for some $t' < T$. We modify the sequence $\{ a_t\}$ by adding the extra value $h$ in the sequence. Precisely, let 
		\[ \tilde{a}_{t} := \begin{cases} a_t  & \text{ if } t\leq t' \\ h & \text{ if } t=t'+1 \\ a_{t-1} & \text{ if } t > t'+1.  \end{cases} \]
		Notice that 
		\[ h=\tilde a_{t'+1} \geq  \tilde{a}_{t'+1} - a_T = \tilde{a}_{t'+1} - \tilde{a}_{t+2} + \tilde{a}_{t'+2}  - a_T  \geq  \tilde{a}_{t'+2}  - a_T = a_{t'+1} - a_T \geq B \delta'^2 (T - (t'+1))  \]
		where the second inequality follows from the fact that, by construction, $\tilde{a}_{t'+1} - \tilde{a}_{t'+2} \geq 0$. Likewise, we have
		\[ 
		\frac{1}{h} = \frac{1}{\tilde a_{t+1}} \geq \frac{1}{a_{t'}} - \frac{1}{a_1}  \geq A (t'-1).  
		\]
		From the above it follows that
		\[ T -2 = (T- (t'+1)) +t'-1 \leq  \frac{1}{A h } + \frac{h}{B \delta'^2}.     \]
		It remains to consider the case $h \geq a_1$. In this case
		\[  h \geq a_1 \geq a_1- a_T \geq B\delta'^2 (T-1) \]
		from where it follows that
		\[  T-2 \leq T-1 \leq  \frac{h}{B\delta'^2}\leq \frac{1}{Ah} +  \frac{h}{B\delta'^2}   \]
		in this case as well.

	\end{proof}

	With Propositions \ref{prop:LowerBoundGap} and \ref{prop : upper bound of multi sinkhorn} and the above lemma in hand, we are ready to prove Theorem~\ref{thm : iteration}.
	\begin{proof}[Proof of Theorem \ref{thm : iteration}]

		Let $\Delta^t := \G^L(g^t) - \G^L(g^*)$. Notice that, thanks to Proposition \ref{prop:LowerBoundGap}, the sequence $\Delta_t$ is decreasing in $t$. Let us denote by $T$ the iteration at which the stopping criterion for Algorithm \ref{alg : multi-sinkhorn} is met. Notice that $T$ is indeed finite, as can be easily verified from Proposition \ref{prop:LowerBoundGap}.
		
		Let $t_1, t_2, t_3, \dots$ be the iterations in $\T$, where we recall $\T$ was defined in Proposition \ref{prop:LowerBoundGap}, and let $t_s$ be the largest element in $\T$ that is strictly smaller than $T$. If such element does not exist, it follows that all iterations before stopping are not in $\T$, but in that case we would have
		\[ T -1 \leq \lceil \frac{\G^L(g^0) }{\min_{i\in \Y}\lVert\mu_i\rVert_1} \rceil, \]
		since the decrement of energy at each of these iterations is at least $\min_{i\in \Y} \lVert \mu_i \rVert_1$. If $t_s$ does exist, by a similar reasoning as before there must also be a first next iteration $t_{s+1}$ in $\T$ (although larger than or equal to $T$). Now, for any $r \leq s$ we have 
		\begin{equation}
		\Delta^{t_r} - \Delta^{t_{r+1}} \geq \Delta^{t_r} - \Delta^{t_{r}+1} \geq \frac{1}{7} \left\{ \left( \frac{\delta'}{K}\right)^2 \vee \left( \frac{\eta \Delta^{t_r}}{K \overline{R}}\right)^2 \right\}. 
		\label{eqn:AuxGap}
		\end{equation}
		Indeed, the first inequality follows from the fact that $\Delta_t$ is decreasing in $t$, and the second inequality follows from the fact that
		\[  \Delta^{t_r} - \Delta^{t_r +1} \geq   \frac{1}{7}\left(\frac{E_{t_r}}{K}\right)^2 \geq \frac{1}{7} \left( \frac{\eta\Delta^{t_r} }{K \overline{R}} \right)^2,  \]
		thanks to Propositions \ref{prop:LowerBoundGap} and \ref{prop : upper bound of multi sinkhorn}. We can thus apply Lemma \ref{lem:NiceSequence} to the sequence $\Delta^{t_1}, \dots, \Delta^{t_{s+1}}$ and deduce that
		\begin{align*}
		s+1 &\leq \min_{ h \text{ s.t. } 
			h \geq \Delta^{t_{s+1}} } \left\{ 2 + 7 \frac{K^2 \overline{R}^2}{\eta^2 h} + 7 h  \left( \frac{K}{\delta'}\right)^2 \right\}.
		\end{align*}
		From the definition of $t_s$, we deduce that $\Delta_{t_{s+1}} \leq \Delta_{T} \leq \frac{ \overline{R}}{\eta} E_T \leq \frac{ \overline{R}}{\eta} \delta'.$ Therefore, taking $h:= \frac{ \overline{R}}{\eta} \delta' $ we obtain 
		\[s +1  \leq 2 + \frac{14K^2 \overline{R}}{\eta \delta'}.    \]

		Finally, the number of iterations not in $\T$ before the stopping criterion is met satisfies
		\[ T-(s+1) \leq \lceil \frac{\G^L(g^0) }{\min_{i\in \Y}\lVert\mu_i\rVert_1} \rceil, \]
		since, again, the decrement of energy at each of these iterations is at least $\min_{i\in \Y} \lVert \mu_i \rVert_1$. The desired estimate on $T$ now follows from the previous two inequalities.
	\end{proof}

	\subsection{Analysis of Round Scheme (Algorithm \ref{alg : Round})}

	\begin{proof}[Proof of Theorem \ref{thm : error analysis of round}]
		Notice that the $\pi^{(i)}_A$'s are non-negative for all $i$. In addition, from the definitions of the $z_i$'s (in particular also the fact that they are less than or equal to one) and the $\pi^{(i)}_{A}$'s we get
		\[
		\text{err}_i := \mu_i - \sum_{A \in S^L_K(i)} {\mathcal{P}_i}_{\#} \pi^{(K)}_A \geq 0.
		\]
		Hence, the $\widehat{\pi}_A$'s are non-negative. Furthermore, the collection $\{ \widehat{\pi}_A :A \in S^L_K\}$ satisfies the marginal constraints, since for each $i \in \mathcal{Y}$ we have
		\[
		\sum_{A \in S^L_K(i)} {\mathcal{P}_i}_{\#} \widehat{\pi}_A = \sum_{A \in S^L_K(i)} {\mathcal{P}_i}_{\#} \pi^{(K)}_A + \text{err}_i = \mu_i.
		\]

		In what follows, we obtain an upper bound on the $\ell^1$ distance between the $\pi_A$'s and the $\widehat{\pi}_A$'s. Letting $\pi^{(0)}_A = \pi_A$, the difference between the mass of $\pi_A$'s and $\pi^{(K)}_A$'s can be written using a telescoping sum as follows:
		\[
		\sum_{A \in S^L_K} \left( ||\pi_A||_1 - ||\pi^{(K)}_A||_1 \right)= \sum_{i=1}^K \sum_{A \in S^L_K} \left(||\pi^{(i-1)}_A||_1 - ||\pi^{(i)}_A||_1 \right).
		\]
		Since $\pi_A^{(1)}= \pi_A$ when $1 \not \in A$, a direct computation yields
		\begin{align*}
		&\sum_{A \in S^L_K}( ||\pi_A||_1 - ||\pi^{(1)}_A||_1)\\
		&=\sum_{A \in S^L_K(1)} \sum_{x_A} \pi_A(x_A) - \sum_{x_1} \left( 1 \wedge \frac{\mu_1(x_1)}{\sum_{A \in S^L_K(1)} {\mathcal{P}_1}_{\#} \pi_A(x_1)} \right) \sum_{x_{A \setminus\{1\}}} \pi_A(x_A)\\
		&=\sum_{A \in S^L_K(1)} \sum_{x_1} \frac{1}{\sum_{A \in S^L_K(1)} {\mathcal{P}_1}_{\#} \pi_A(x_1)} \left( \left\{ \sum_{A \in S^L_K(1)} {\mathcal{P}_1}_{\#} \pi_A(x_1) - \mu_1(x_1) \right\} \vee 0 \right) \sum_{x_{A \setminus \{1\}}} \pi_A(x_A)\\
		&= \sum_{x_1} \frac{1}{\sum_{A \in S^L_K(1)} {\mathcal{P}_1}_{\#} \pi_A(x_1)}  \left( \left\{ \sum_{A \in S^L_K(1)} {\mathcal{P}_1}_{\#} \pi_A(x_1) - \mu_1(x_1) \right\} \vee 0 \right) \sum_{A \in S^L_K(1)} \sum_{x_{A \setminus \{1\}}} \pi_A(x_A)\\
		&= \sum_{x_1}  \left( \left\{ \sum_{A \in S^L_K(1)} {\mathcal{P}_1}_{\#} \pi_A(x_1) - \mu_1(x_1) \right\} \vee 0 \right)\\
		&= \frac{1}{2} \left( || \sum_{A \in S^L_K(1)} {\mathcal{P}_1}_{\#} \pi_A - \mu_1 ||_1 + || \sum_{A \in S^L_K(1)} {\mathcal{P}_1}_{\#} \pi_A ||_1 - || \mu_1 ||_1 \right).
		\end{align*}
		In addition, since $z_i(x_i) \leq 1$ for all $i \in \mathcal{Y}$ and all $x_i$, it follows that for all $i \in \mathcal{Y}$,
		\[
		{\mathcal{P}_i}_{\#} \pi^{(K)}_A \leq \dots \leq {\mathcal{P}_i}_{\#} \pi^{(0)}_A = {\mathcal{P}_i}_{\#} \pi_A.
		\]
		Hence, similarly as above,
		\begin{align}
		\label{eqn:AuxRound1}
		\begin{split}
		\sum_{A \in S^L_K} \left( ||\pi^{(i-1)}_A||_1 - ||\pi^{(i)}_A||_1 \right)&= \sum_{x_i}  \left( \left\{ \sum_{A \in S^L_K(i)} {\mathcal{P}_i}_{\#} \pi^{(i-1)}_A(x_i) - \mu_i(x_i) \right\} \vee 0 \right)\\
		&\leq \sum_{x_i}  \left( \left\{ \sum_{A \in S^L_K(i)} {\mathcal{P}_i}_{\#} \pi^{(0)}_A(x_i) - \mu_i(x_i) \right\} \vee 0 \right)\\
		&= \frac{1}{2} \left( || \sum_{A \in S^L_K(i)} {\mathcal{P}_i}_{\#} \pi_A - \mu_i ||_1 + || \sum_{A \in S^L_K(i)} {\mathcal{P}_i}_{\#} \pi_A ||_1 - || \mu_i ||_1 \right).
		\end{split}   
		\end{align}
		As a result,
		\begin{equation}\label{eq : upper bound pi0 piK}
		\begin{aligned}
		&\sum_{A \in S^L_K} \left( ||\pi_A||_1 - ||\pi^{(K)}_A||_1\right)\\
		&\leq \frac{1}{2} \sum_{i \in \mathcal{Y}} \left( || \sum_{A \in S^L_K(i)} {\mathcal{P}_i}_{\#} \pi_A - \mu_i ||_1 + || \sum_{A \in S^L_K(i)} {\mathcal{P}_i}_{\#} \pi_A ||_1 - || \mu_i ||_1 \right).
		\end{aligned}       
		\end{equation}
		On the other hand, from \eqref{eqn:AuxRound1} we also get
		\begin{equation}\label{eq : loose upper bound pi0 piK}
		\sum_{A \in S^L_K} \left( ||\pi^{(i-1)}_A||_1 - ||\pi^{(i)}_A||_1\right)  \leq || \sum_{A \in S^L_K(i)} {\mathcal{P}_i}_{\#} \pi_A - \mu_i ||_1.
		\end{equation}

		Recalling the definition of $\widehat{\pi}_A$'s and using the facts that $\mu_i \geq \sum_{A \in S^L_K(i)} {\mathcal{P}_i}_{\#} \pi^{(K)}_A$ for all $i \in \mathcal{Y}$ and $\pi_A \geq \pi^{(K)}_A$ for all $A \in S^L_K$, it follows that
		\begin{align*}
		&\sum_{A \in S^L_K} || \widehat{\pi}_A - \pi_A ||_1\\
		&\leq \sum_{A \in S^L_K} || \widehat{\pi}_A - \pi^{(K)}_A ||_1 + \sum_{A \in S^L_K} || \pi^{(K)}_A - \pi_A ||_1\\
		&  = \sum_{i \in \Y} || \widehat{\pi}_{\{i\}} - \pi^{(K)}_{\{i\}} ||_1 + \sum_{A \in S^L_K} || \pi^{(K)}_A - \pi_A ||_1\\
		&  = \sum_{i \in \Y} || \text{err}_i ||_1 + \sum_{A \in S^L_K} || \pi^{(K)}_A - \pi_A ||_1\\
		&= \sum_{i \in \mathcal{Y}}( ||\mu_i||_1 - ||\sum_{A \in S^L_K(i)} {\mathcal{P}_i}_{\#} \pi^{(K)}_A ||_1 ) + \sum_{A \in S^L_K} ||\pi_A ||_1 - || \pi^{(K)}_A||_1\\
		&=\sum_{i \in \mathcal{Y}} (||\mu_i||_1 - ||\sum_{A \in S^L_K(i)} {\mathcal{P}_i}_{\#} \pi_A ||_1)+ \underbrace{ \sum_{i \in \mathcal{Y}} (||\sum_{A \in S^L_K(i)} {\mathcal{P}_i}_{\#} \pi_A ||_1 - ||\sum_{A \in S^L_K(i)} {\mathcal{P}_i}_{\#} \pi^{(K)}_A ||_1) }_{\text{I}}\\
		&\quad + \underbrace{ \sum_{A \in S^L_K} (||\pi_A ||_1 - || \pi^{(K)}_A||_1) }_{\text{II}}.
		\end{align*}
		Let's consider the $\text{I}$ term first. Note that for each $A \in S^L_K$, $\pi_A$ and $\pi^{(K)}_A$ appear at most $|A|\leq L$ times in the sum and we also have $\pi_A \geq \pi^{(K)}_A$. Thus, by \eqref{eq : upper bound pi0 piK} and \eqref{eq : loose upper bound pi0 piK} we obtain
		\begin{align*}
		\text{I} = &\sum_{i \in \mathcal{Y}}(||\sum_{A \in S^L_K(i)} {\mathcal{P}_i}_{\#} \pi_A ||_1 - ||\sum_{A \in S^L_K(i)} {\mathcal{P}_i}_{\#} \pi^{(K)}_A ||_1)\\
		&= \sum_{i \in \mathcal{Y}} \sum_{A \in S^L_K(i)} \sum_{x_A \in \mathcal{X}^A} \pi_A(x_A) - \sum_{i \in \mathcal{Y}} \sum_{A \in S^L_K(i)} \sum_{x_A \in \mathcal{X}^A} \pi^{(K)}_A(x_A)\\
		&= \sum_{A \in S^L_K} \sum_{i \in A} \sum_{x_A \in \mathcal{X}^A} \pi_A(x_A) - \sum_{A \in S^L_K} \sum_{i \in A} \sum_{x_A \in \mathcal{X}^A} \pi^{(K)}_A(x_A)\\
		& \leq L \sum_{A\in S_K^L} ( || \pi_A||_1 - ||\pi_A^{(K)} ||_1 ) \\
		&= (L-2) \sum_{A \in S^L_K} \left( ||\pi_A ||_1 - || \pi^{(K)}_A||_1 \right) + 2 \sum_{A \in S^L_K} \left( ||\pi_A ||_1 - || \pi^{(K)}_A||_1 \right)\\
		&\leq (L-1) \sum_{i \in \mathcal{Y}} || \sum_{A \in S^L_K(i)} {\mathcal{P}_i}_{\#} \pi_A - \mu_i ||_1 +  \sum_{i \in \mathcal{Y}} ( || \sum_{A \in S^L_K(i)} {\mathcal{P}_i}_{\#} \pi_A ||_1 - || \mu_i ||_1).
		\end{align*}
		Applying \eqref{eq : loose upper bound pi0 piK} to $\text{II}$ similarly, we deduce
		\begin{align*}
		\sum_{A \in S^L_K} || \widehat{\pi}_A - \pi_A ||_1 &\leq \sum_{i \in \mathcal{Y}} ||\mu_i||_1 - ||\sum_{A \in S^L_K(i)} {\mathcal{P}_i}_{\#} \pi_A ||_1\\
		&\quad + (L-1) \sum_{i \in \mathcal{Y}}  || \sum_{A \in S^L_K(i)} {\mathcal{P}_i}_{\#} \pi_A - \mu_i ||_1 + \sum_{i \in \mathcal{Y}} || (\sum_{A \in S^L_K(i)} {\mathcal{P}_i}_{\#} \pi_A ||_1 - || \mu_i ||_1) \\
		&\quad + \sum_{i \in \mathcal{Y}} || \sum_{A \in S^L_K(i)} {\mathcal{P}_i}_{\#} \pi_A - \mu_i ||_1\\
		&= L \sum_{i \in \mathcal{Y}} || \sum_{A \in S^L_K(i)} {\mathcal{P}_i}_{\#} \pi_A - \mu_i ||_1.
		\end{align*}
		This completes the proof.
	\end{proof}

	\subsection{Analysis of Algorithm \ref{alg : entropic regularization}}

	\begin{proof}[Proof of Theorem \ref{thm:AnalysisALgoWithRound}]
		Let $\{\widetilde{\pi}_A : A \in S^L_K \}$ be an output of \textbf{Step 2} of Algorithm \ref{alg : entropic regularization}. Given that $c_{\{ i\}} \equiv 0$ for each $i$, we deduce that for every $A \in S^L_K$ we have $\spt(\widetilde{\pi}_A) \subseteq \{c_A < \infty \}$ and in turn also $\spt(\widehat{\pi}_A) \subseteq \{c_A < \infty \}$. By Theorem \ref{thm : error analysis of round},
		\begin{equation}\label{eq : pi hat minus pi tilde}
		\begin{aligned}
		&\sum_{A \in S^L_K} \sum_{x_A \in \mathcal{X}^A} (1 + c_A(x_A)) \left( \widehat{\pi}_A(x_A) - \widetilde{\pi}_A(x_A) \right)\\
		&\leq L \left(1 + \max_{A \in S^L_K} |c_A \mathds{1}_{c_A < \infty} | \right) \sum_{i=1}^K ||\sum_{A \in S^L_K(i)} {\mathcal{P}_i}_{\#} \widetilde{\pi}_A - \mu_i ||_1.
		\end{aligned}   
		\end{equation}
		Let $\{ \pi^*_A : A \in S^L_K \}$ be a set of optimal couplings for
		\begin{equation*}
		\begin{aligned}
		&\min_{\{\pi_A:A \in S^L_K\}} \sum_{A \in S^L_K} \sum_{\mathcal{X}^A} \left ( 1 + c_A(x_A) \right )\pi_A(x_A)   \text{ s.t. } \sum_{A \in S^L_K(i)} {\mathcal{P}_i}_{\#} \pi_A = \mu_i \hspace{0.5cm} \text{ for all } i \in \mathcal{Y}
		\end{aligned}
		\end{equation*}
		and $\{ \pi'_A : A \in S^L_K \}$ be an output of Algorithm \ref{alg : Round} with input $\{ \pi^*_A : A \in S^L_K \}$ and $\nu:=(\nu_1, \dots, \nu_K)$, where
		\[
		\nu_i := \sum_{A \in S^L_K(i)} {\mathcal{P}_i}_{\#} \widetilde{\pi}_A.
		\]   
		Then, since $\sum_{A \in S^L_K(i)} {\mathcal{P}_i}_{\#} \pi^*_A = \mu_i$ for all $i \in \mathcal{Y}$,
		\begin{equation}\label{eq : pi prime minus pi star}
		\sum_{A \in S^L_K} || \pi'_A - \pi^*_A ||_1 \leq L \sum_{i=1}^K ||\nu_i - \mu_i ||_1.
		\end{equation}
		Note that $\{\widetilde{\pi}_A : A \in S^L_K \}$ is a solution for \eqref{eq : truncated entropic problem} when $\mu$ is replaced by $\nu$. Since $\{ \pi'_A : A \in S^L_K \}$ is also feasible for this problem,
		\[
		\sum_{A \in S^L_K} \sum_{x_A \in \mathcal{X}^A} (1 + c_A(x_A)) \left( \widetilde{\pi}_A(x_A) - \pi'_A(x_A) \right) \leq  \sum_{A \in S^L_K} \eta H(\widetilde{\pi}_A) - \eta H(\pi'_A).
		\]
		Recall the log sum inequality: for $(a_1, \dots, a_n), (b_1, \dots, b_n) \in \mathbb{R}^n_+$,
		\begin{equation}\label{eq: log sum inequality}
		\sum_{i=1}^n a_i \log \frac{a_i}{b_i} \geq \left(\sum_{i=1}^n a_i \right) \log \frac{ \left( \sum_{i=1}^n a_i \right)}{\left( \sum_{i=1}^n b_i \right)}.
		\end{equation}
		Using the fact that $H(\pi_A') \geq 0$ (since $\pi_A' \leq 1$ for all $A$) 
		and applying \eqref{eq: log sum inequality} we obtain 
		\begin{align*}
		&\sum_{A \in S^L_K} H(\widetilde{\pi}_A) - H(\pi'_A)\\
		& \leq \sum_{A \in S^L_K} H(\widetilde{\pi}_A)\\
		& = \sum_{A \in S^L_K} \sum_{x_A \in \mathcal{X}^A} \left(1 - \log \widetilde{\pi}_A(x_A) \right) \widetilde{\pi}_A(x_A)\\
		&\leq \left( \sum_{A \in S^L_K} \sum_{x_A \in \mathcal{X}^A} \widetilde{\pi}_A(x_A) \right) - \left( \sum_{A \in S^L_K} \sum_{x_A \in \mathcal{X}^A} \widetilde{\pi}_A(x_A) \right) \log \frac{\left( \sum_{A \in S^L_K} \sum_{x_A \in \mathcal{X}^A} \widetilde{\pi}_A(x_A) \right)}{\left( \sum_{A \in S^L_K} \sum_{x_A \in \mathcal{X}^A} 1 \right)}\\
		&=\left( \sum_{A \in S^L_K} \sum_{x_A \in \mathcal{X}^A} \widetilde{\pi}_A(x_A) \right)- \left( \sum_{A \in S^L_K} \sum_{x_A \in \mathcal{X}^A} \widetilde{\pi}_A(x_A) \right) \log\left( \sum_{A \in S^L_K} \sum_{x_A \in \mathcal{X}^A} \widetilde{\pi}_A(x_A) \right)\\
		&\quad + \left( \sum_{A \in S^L_K} \sum_{x_A \in \mathcal{X}^A} \widetilde{\pi}_A(x_A) \right) \log \left( \sum_{A \in S^L_K} \sum_{x_A \in \mathcal{X}^A} 1 \right)\\
		&\leq 1 +  \left( \sum_{A \in S^L_K} \sum_{x_A \in \mathcal{X}^A} \widetilde{\pi}_A(x_A) \right) \log \left( \sum_{A \in S^L_K} \sum_{x_A \in \mathcal{X}^A} 1 \right)\\
		&\leq 1 + \frac{3L}{2}\log(C^*Kn),
		\end{align*}
		where the second to last inequality follows from the fact that $x -x \log x$ is bounded from above by $1$ and the last inequality is a consequence of 
		\[
		\sum_{A \in S^L_K} || \widetilde{\pi}_A||_1 \leq \sum_{i \in \mathcal{Y}} ||\nu_i||_1 \leq \sum_{i \in \mathcal{Y}} ||\nu_i - \mu_i||_1 + ||\mu_i||_1 \leq \frac{3}{2},
		\]
		due to the stopping criterion $E_t := \sum_{i \in \mathcal{Y}} ||\nu_i - \mu_i||_1 < \delta'\leq \frac{1}{2}$. Hence,
		\begin{equation}\label{eq : pi tilde minus pi prime}
		\sum_{A \in S^L_K} \sum_{x_A \in \mathcal{X}^A} (1 + c_A(x_A)) \left( \widetilde{\pi}_A(x_A) - \pi'_A(x_A) \right) \leq  \eta  2L\log(C^*Kn) .
		\end{equation}
		Combining \eqref{eq : pi prime minus pi star} and \eqref{eq : pi tilde minus pi prime} leads to
		\begin{align}
		&\sum_{A \in S^L_K} \sum_{\mathcal{X}^A} \left ( 1 + c_A(x_A) \right ) \left(\widetilde{\pi}_A(x_A) - \pi^*_A(x_A) \right) \nonumber \\
		&=  \sum_{A \in S^L_K} \sum_{\mathcal{X}^A} \left ( 1 + c_A(x_A) \right ) \left(\widetilde{\pi}_A(x_A) - \pi'_A(x_A) \right) + \sum_{A \in S^L_K} \sum_{\mathcal{X}^A} \left ( 1 + c_A(x_A) \right ) \left(\pi'_A(x_A) - \pi^*_A(x_A) \right) \nonumber\\
		&\leq \eta  2L\log(C^*Kn)  + L \left(1 + \max_{A \in S^L_K} |c_A \mathds{1}_{c_A < \infty} | \right) \sum_{i=1}^K ||\nu_i - \mu_i ||_1. \label{eq : pi tilde minus pi star}
		\end{align}
		Lastly, combining \eqref{eq : pi hat minus pi tilde} and \eqref{eq : pi tilde minus pi star} leads to
		\begin{equation}\label{eq: error estimate}
		\begin{aligned}
		&\sum_{A \in S^L_K} \sum_{\mathcal{X}^A} \left ( 1 + c_A(x_A) \right ) \left(\widehat{\pi}_A(x_A) - \pi^*_A(x_A) \right)\\
		&\leq \eta  2L\log(C^*Kn)  + 2L \left(1 + \max_{A \in S^L_K} |c_A \mathds{1}_{c_A < \infty} | \right) \sum_{i=1}^K ||\sum_{A \in S^L_K(i)} {\mathcal{P}_i}_{\#} \widetilde{\pi}_A - \mu_i ||_1.
		\end{aligned}   
		\end{equation}

		Recall the choices of $\eta$ and $\delta'$:
		\begin{align*}
		\eta = \frac{\delta / 2}{  2L\log(C^*Kn) }, \quad  \delta' = \frac{\delta / 2}{2L \max_{A \in S^L_K} | 	1 + c_A\mathds{1}_{c_A < \infty} |}.
		\end{align*}
		Using the definition of $\widetilde{\mu}=(\widetilde{\mu}_1, \dots, \widetilde{\mu}_K)$ and the fact that
		\[
		\sum_{i \in \mathcal{Y}} || \widetilde{\mu}_i -     \sum_{A \in S^L_K(i)} {\mathcal{P}_i}_{\#} \widetilde{\pi}_A ||_1 \leq \frac{\delta'}{2},
		\] 
		we obtain
		\[
		\sum_{i \in \mathcal{Y}} || \mu_i -     \sum_{A \in S^L_K(i)} {\mathcal{P}_i}_{\#} \widetilde{\pi}_A ||_1 \leq\sum_{i \in \mathcal{Y}} ||\mu_i -  \widetilde{\mu}_i||_1 + || \widetilde{\mu}_i - \sum_{A \in S^L_K(i)} {\mathcal{P}_i}_{\#} \widetilde{\pi}_A ||_1 \leq \delta'.
		\]
		Applying the above inequality, and using our choices of $\eta$ and $\delta'$ in \eqref{eq: error estimate}, we obtain
		\[
		\sum_{A \in S^L_K} \sum_{\mathcal{X}^A} \left ( 1 + c_A(x_A) \right ) \left(\widehat{\pi}_A(x_A) - \pi^*_A(x_A) \right) \leq \delta.
		\]

		Now, it remains to bound the computational complexity of Algorithm \ref{alg : entropic regularization}. Using \eqref{eq : number of iteration} and the definition of $\overline{R}$ we see that Algorithm 2 in Step 2 requires at most $T$ iterations to stop, where 
				\begin{align*}
				T &\leq O(1) + \frac{14 K^2 \overline{R}}{ \eta \delta'}\\
				&\leq O(1) + O \left(\frac{   L^2 K^2 \max_{A \in S_K^L} (1 + c_A \mathds{1}_{c_A < \infty} )   \log (C^* Kn)   }{ \delta^2}\right).
				\end{align*}
		
Since each iteration of Algorithm \ref{alg : multi-sinkhorn} requires at most $O(\mathcal{I}_L)$ operations, the total computational complexity of \textbf{Step 2} of Algorithm \ref{alg : entropic regularization} is $O \left(  \frac{   L^2 K^{2} \max_{A \in S_K^L} (1 + c_A \mathds{1}_{c_A < \infty} ) |\mathcal{I}_L|   \log (C^* Kn)   }{ \delta^2}  \right)$. \textbf{Step 1} and \textbf{Step 3} of Algorithm \ref{alg : entropic regularization} require $O(Kn)$ and $O(\mathcal{I}_L)$ operations, respectively. Therefore, the conclusion follows.
	\end{proof}

	
	
	


\bibliography{Ref.bib}
\bibliographystyle{amsalpha}
\end{document}